\documentclass[twoside,11pt,letterpaper]{article}
\usepackage{jair/jair}
\usepackage{jair/theapa}
\usepackage{amsmath}
\usepackage{amssymb}
\usepackage{amsthm}
\usepackage{graphicx}
\usepackage{algo}
\usepackage{psfrag}

\jairheading{29}{2007}{153-190}{10/06}{06/07}
\ShortHeadings{The Generalized A* Architecture}{Felzenszwalb \& McAllester}
\firstpageno{153}

\newtheorem{theorem}{\noindent Theorem}
\newtheorem{lemma}[theorem]{\noindent Lemma}

\newtheorem{definition}{\noindent Definition}

\newcommand{\map}[3]{\mbox{$#1\!:\!#2\! \rightarrow\!#3$}}
\newcommand{\path}[1]{\text{\emph{path}}(#1)}
\newcommand{\phrase}[1]{\text{\emph{phrase}}(#1)}
\newcommand{\curve}[1]{\text{\emph{curve}}(#1)}
\newcommand{\seg}[1]{\text{\emph{seg}}(#1)}
\newcommand{\shape}[1]{\text{\emph{shape}}(#1)}
\newcommand{\convex}[1]{\text{\emph{convex}}(#1)}
\newcommand{\goal}{\text{\emph{goal}}}
\newcommand{\context}{\text{\emph{context}}}
\newcommand{\abs}{\text{\emph{abs}}}
\newcommand{\true}{\text{\bf true}}

\newcommand{\onerule}[1]{
  \parbox{2.0in}{#1}}

\def\irule#1#2#3#4{
\begin{tabbing}
#2 \parbox{.75in}{\noindent \hrule ~} $#3$ \\ #4
\end{tabbing}}
\def\ant#1{$#1$ \\}
\def\con#1{$#1$}

\begin{document}

\title{The Generalized A* Architecture}

\author{\name Pedro F. Felzenszwalb \email pff@cs.uchicago.edu \\
  \addr Department of Computer Science \\
  \addr University of Chicago \\
  \addr Chicago, IL 60637 
  \AND
  \name David McAllester \email mcallester@tti-c.org \\
  \addr Toyota Technological Institute at Chicago \\
  \addr Chicago, IL 60637}

\maketitle

\begin{abstract}
 We consider the problem of computing a lightest derivation of a global
 structure using a set of weighted rules.  A large variety of
 inference problems in AI can be formulated in this framework.  We
 generalize A* search and heuristics derived from abstractions to a
 broad class of lightest derivation problems.  We also describe a new
 algorithm that searches for lightest derivations using a hierarchy of
 abstractions.  Our generalization of A* gives a new algorithm for
 searching AND/OR graphs in a bottom-up fashion.
  
 We discuss how the algorithms described here provide a general
 architecture for addressing the pipeline problem --- the problem of
 passing information back and forth between various stages of
 processing in a perceptual system.  We consider examples in computer
 vision and natural language processing.  We apply the hierarchical
 search algorithm to the problem of estimating the boundaries of
 convex objects in grayscale images and compare it to other search
 methods.  A second set of experiments demonstrate the use of a new
 compositional model for finding salient curves in images.
\end{abstract}

\section{Introduction}

We consider a class of problems defined by a set of weighted rules for
composing structures into larger structures.  The goal in such
problems is to find a lightest (least cost) derivation of a global
structure derivable with the given rules.  A large variety of
classical inference problems in AI can be expressed within this
framework.  For example the global structure might be a parse tree, a
match of a deformable object model to an image, or an assignment of
values to variables in a Markov random field.

We define a \emph{lightest derivation problem} in terms of a set of
statements, a set of weighted rules for deriving statements using
other statements and a special goal statement.  In each case we are
looking for the lightest derivation of the goal statement.  We usually
express a lightest derivation problem using rule ``schemas'' that
implicitly represent a very large set of rules in terms of a small
number of rules with variables.  Lightest derivation problems are
formally equivalent to search in AND/OR graphs \cite{Nilsson80}, but
we find that our formulation is more natural for the 
applications we are interested in.  

One of the goals of this research is the construction of
algorithms for global optimization across many levels of processing in
a perceptual system.  As described below our algorithms can be used to
integrate multiple stages of a processing pipeline into a single
global optimization problem that can be solved efficiently.

Dynamic programming is a fundamental technique for designing
efficient inference algorithms.  Good examples are the Viterbi
algorithm for hidden Markov models \cite{Rabiner89} and chart parsing
methods for stochastic context free grammars \cite{Charniak}.  The
algorithms described here can be used to speed up the
solution of problems normally solved using dynamic programming.  We
demonstrate this for a specific problem, where the goal is to estimate
the boundary of a convex object in a cluttered image.  In a second set
of experiments we show how our algorithms can be used to find salient
curves in images.  We describe a new model for salient curves based on
a compositional rule that enforces long range shape constraints.  This
leads to a problem that is too large to be solved using classical dynamic
programming methods.

The algorithms we consider are all related to Dijkstra's shortest
paths algorithm (DSP) \cite{Dijkstra59} and A* search \cite{Hart68}.
Both DSP and A* can be used to find a shortest path in a cyclic graph.
They use a priority queue to define an order in which nodes are
expanded and have a worst case running time of $O(M \log N)$ where $N$
is the number of nodes in the graph and $M$ is the number of edges.
In DSP and A* the expansion of a node $v$ involves
generating all nodes $u$ such that there is an edge from $v$ to $u$.
The only difference between the two methods is that A* uses a
heuristic function to avoid expanding non-promising nodes.

Knuth gave a generalization of DSP that can be used to solve a
lightest derivation problem with cyclic rules \cite{Knuth77}.  We call
this Knuth's lightest derivation algorithm (KLD).  In analogy to
Dijkstra's algorithm, KLD uses a priority queue to define an order in
which statements are expanded.  Here the expansion of a statement $v$
involves generating all conclusions that can be derived in a single
step using $v$ and other statements already expanded.  As long as each
rule has a bounded number of antecedents KLD also has a worst case
running time of $O(M \log N)$ where $N$ is the number of statements in
the problem and $M$ is the number of rules.  Nilsson's AO* algorithm
\citeyear{Nilsson80} can also be used to solve lightest derivation
problems.  Although AO* can use a heuristic function, it is not a true
generalization of A* --- it does not use a priority queue, only handles
acyclic rules, and can require $O(MN)$ time even when applied to a
shortest path problem.\footnote{There are extensions that handle
cyclic rules \cite{Jimenez00}.} In particular, AO* and its variants use
a backward chaining technique that starts at the goal and repeatedly
refines subgoals, while A* is a forward chaining algorithm.\footnote{
AO* is backward chaining in terms of the inference rules defining a
lightest derivation problem.}

\citeA{Klein03} described an A* parsing algorithm that is similar to KLD
but can use a heuristic function.  One of our
contributions is a generalization of this algorithm to arbitrary
lightest derivation problems.  We call this algorithm A* lightest
derivation (A*LD).  The method is forward chaining, uses a priority
queue to control the order in which statements are expanded, handles
cyclic rules and has a worst case running time of $O(M \log N)$ for
problems where each rule has a small number of antecedents.  A*LD can
be seen as a true generalization of A* to lightest derivation
problems.  For a lightest derivation problem that comes from a
shortest path problem A*LD is identical to A*.

Of course the running times seen in practice are often not well
predicted by worst case analysis.  This is specially true for problems
that are very large and defined implicitly.  For example, we can use
dynamic programming to solve a shortest path problem in an acyclic
graph in $O(M)$ time.  This is better than the $O(M \log N)$ bound for
DSP, but for implicit graphs DSP can be much more efficient since it
expands nodes in a best-first order.  When searching for a shortest
path from a source to a goal, DSP will only expand nodes $v$ with
$d(v) \leq w^*$.  Here $d(v)$ is the length of a shortest path from
the source to $v$, and $w^*$ is the length of a shortest path from the
source to the goal.  In the case of A* with a monotone and admissible
heuristic function, $h(v)$, it is possible to obtain a similar bound
when searching implicit graphs.  A* will only expand nodes $v$ with
$d(v) + h(v) \leq w^*$.

The running time of KLD and A*LD can be expressed in a similar way.
When solving a lightest derivation problem, KLD will only expand
statements $v$ with $d(v) \leq w^*$.  Here $d(v)$ is the weight of a
lightest derivation for $v$, and $w^*$ is the weight of a lightest
derivation of the goal statement.  Furthermore, A*LD will only expand
statements $v$ with $d(v) + h(v) \leq w^*$.  Here the heuristic
function, $h(v)$, gives an estimate of the additional weight necessary
for deriving the goal statement using a derivation of $v$.  The
heuristic values used by A*LD are analogous to the distance from a
node to the goal in a graph search problem (the
notion used by A*).  We note that these heuristic values are
significantly different from the ones used by AO*.  In the case of AO*
the heuristic function, $h(v)$, would estimate the weight of a
lightest derivation for $v$.

An important difference between A*LD and AO* is that A*LD computes
derivations in a bottom-up fashion, while AO* uses a top-down
approach.  Each method has advantages, depending on the type of
problem being solved.  For example, a classical problem in computer
vision involves grouping pixels into long and smooth curves.  We can
formulate the problem in terms of finding smooth curves between pairs
of pixels that are far apart.  For an image with $n$ pixels there are
$\Omega(n^2)$ such pairs.  A straight forward implementation of a
top-down algorithm would start by considering these $\Omega(n^2)$
possibilities.  A bottom-up algorithm would start with $O(n)$ pairs of
nearby pixels.  In this case we expect that a bottom-up grouping
method would be more efficient than a top-down method.

The classical AO* algorithm requires the set of rules to be acyclic.
\citeA{Jimenez00} extended the method to handle cyclic rules.  Another
top-down algorithm that can handle cyclic rules is described by
\citeA{Bonet05}.  \citeA{Hansen01} described a search algorithm for
problems where the optimal solutions themselves can be cyclic.  The
algorithms described in this paper can handle problems with cyclic
rules but require that the optimal solutions be acyclic.  We also note that
AO* can handle rules with non-superior weight functions (as defined in
Section~\ref{sec:KLD}) while KLD requires superior weight functions.
A*LD replaces this requirement by a requirement on the heuristic
function.  

A well known method for defining heuristics for A* is to consider an
abstract or relaxed search problem.  For example, consider the problem
of solving a Rubik's cube in a small number of moves.  Suppose we
ignore the edge and center pieces and solve only the corners.  This is
an example of a problem abstraction.  The number of moves necessary to
put the corners in a good configuration is a lower bound on the number
of moves necessary to solve the original problem.  There are fewer
corner configurations than there are full configurations and that
makes it easier to solve the abstract problem.  In general, shortest
paths to the goal in an abstract problem can be used to define an
admissible and monotone heuristic function for solving the original
problem with A*.  

Here we show that abstractions can also be used to define heuristic
functions for A*LD.  In a lightest derivation problem the notion of a
shortest path to the goal is replaced by the notion of a lightest
context, where a context for a statement $v$ is a derivation of the
goal with a ``hole'' that can be filled in by a derivation
of $v$.  The computation of lightest abstract contexts is itself a
lightest derivation problem.

Abstractions are related to problem relaxations defined by
\citeA{Pearl84}.  While abstractions often lead to small
problems that are solved through search, relaxations can lead to
problems that still have a large state space but may be simple enough
to be solved in closed form.  The definition of abstractions that we
use for lightest derivation problems includes relaxations as a special
case.

Another contribution of our work is a hierarchical search method that
we call HA*LD.  This algorithm can effectively use a hierarchy of
abstractions to solve a lightest derivation problem.  The algorithm is
novel even in the case of classical search (shortest paths) problem.
HA*LD searches for lightest derivations and contexts at every level of
abstraction simultaneously.  More specifically, each level of
abstraction has its own set of statements and rules.  The search for
lightest derivations and contexts at each level is controlled by a
single priority queue.  To understand the running time of HA*LD, let
$w^*$ be the weight of a lightest derivation of the goal in the
original (not abstracted) problem.  For a statement $v$ in the
abstraction hierarchy let $d(v)$ be the weight of a lightest
derivation for $v$ at its level of abstraction.  Let $h(v)$ be the
weight of a lightest context for the abstraction of $v$ (defined at
one level above $v$ in the hierarchy).  Let $K$ be the total number of
statements in the hierarchy with $d(v) + h(v) \le w^*$.  HAL*D expands
at most $2K$ statements before solving the original
problem.  The factor of two comes from the fact that the algorithm
computes both derivations and contexts at each level of abstraction.

Previous algorithms that use abstractions for solving search problems
include methods based on pattern databases
\cite{Culberson98,Korf97,Korf02}, Hierarchical A* (HA*, HIDA*)
\cite{Holte96,Holte05} and coarse-to-fine dynamic programming (CFDP)
\cite{Raphael01}.  Pattern databases have made it possible to compute
solutions to impressively large search problems.  These
methods construct a lookup table of shortest paths from a node to the
goal at all abstract states.  In practice the approach is limited to
tables that remain fixed over different problem instances, or
relatively small tables if the heuristic must be recomputed for each
instance.  For example, for the Rubik's cube we can precompute the
number of moves necessary to solve every corner configuration.  This
table can be used to define a heuristic function when solving any full
configuration of the Rubik's cube.  Both HA* and HIDA* use a hierarchy
of abstractions and can avoid searching over all nodes at any level of
the hierarchy.  On the other hand, in directed graphs these methods
may still expand abstract nodes with arbitrarily large heuristic
values.  It is also not clear how to generalize HA* and HIDA* to
lightest derivation problems that have rules with more than one
antecedent.  Finally, CFDP is related to AO* in that it repeatedly
solves ever more refined problems using dynamic programming.  This
leads to a worst case running time of $O(NM)$.  We will discuss the
relationships between HA*LD and these other hierarchical methods in
more detail in Section~\ref{sec:hierarchical}.  

We note that both A* search and related algorithms have been
previously used to solve a number of problems that are not classical
state space search problems.  This includes the traveling salesman
problem \cite{Zhang96}, planning \cite{Edelkamp02}, multiple sequence
alignment \cite{Korf05}, combinatorial problems on graphs
\cite{Felner05} and parsing using context-free-grammars
\cite{Klein03}.  The work by \citeA{Bulitko06} uses a hierarchy of
state-space abstractions for real-time search.

\subsection{The Pipeline Problem}

A major problem in artificial intelligence is the integration of
multiple processing stages to form a complete perceptual system.  We
call this the \emph{pipeline problem}.  In general we have a
concatenation of systems where each stage feeds information to the
next.  In vision, for example, we might have an edge detector feeding
information to a boundary finding system, which in turn feeds
information to an object recognition system.


Because of computational constraints and the need to build modules
with clean interfaces pipelines often make hard decisions at module
boundaries.  For example, an edge detector typically constructs a
Boolean array that indicates weather or not an edge was detected at
each image location.  But there is general recognition that the
presence of an edge at a certain location can depend on the context
around it.  People often see edges at places where the image gradient
is small if, at higher cognitive level, it is clear that there is
actually an object boundary at that location.  Speech recognition
systems try to address this problem by returning n-best lists, but
these may or may not contain the actual utterance.  We would like the
speech recognition system to be able to take high-level information
into account and avoid the hard decision of exactly what strings to
output in its n-best list.

A processing pipeline can be specified by describing each of its
stages in terms of rules for constructing structures using structures
produced from a previous stage.  In a vision system one stage could
have rules for grouping edges into smooth curves while the next stage
could have rules for grouping smooth curves into objects.  In this
case we can construct a single lightest derivation problem
representing the entire system.  Moreover, a hierarchical set of
abstractions can be applied to the entire pipeline.  By using HA*LD to
compute lightest derivations a complete scene interpretation derived
at one level of abstraction guides all processing stages at a more
concrete level.  This provides a mechanism that enables coarse
high-level processing to guide low-level computation.  We believe that
this is an important property for implementing efficient perceptual
pipelines that avoid making hard decisions between processing stages.

We note that the formulation of a complete computer vision system as a
lightest derivation problem is related to the work by \citeA{Geman02},
\citeA{Tu05} and \citeA{Jin06}.  In these papers image understanding is
posed as a parsing problem, where the goal is to explain the image in
terms of a set of objects that are formed by the (possibly recursive)
composition of generic parts.  \citeA{Tu05} use data
driven MCMC to compute ``optimal'' parses while 
\citeA{Geman02} and \citeA{Jin06} use a bottom-up algorithm for building
compositions in a greedy fashion.  Neither of these methods are
guaranteed to compute an optimal scene interpretation.  We hope that
HA*LD will provide a more principled computational technique for
solving large parsing problems defined by compositional models.

\subsection{Overview}

We begin by formally defining lightest derivation problems in
Section~\ref{sec:LDP}.  That section also discusses dynamic
programming and the relationship between lightest derivation problems
and AND/OR graphs.  In Section~\ref{sec:KLD} we describe Knuth's
lightest derivation algorithm.  In Section~\ref{sec:ALD} we describe
A*LD and prove its correctness.  Section~\ref{sec:abs} shows how
abstractions can be used to define mechanically constructed heuristic
functions for A*LD.  We describe HA*LD in Section~\ref{sec:HALD} and
discuss its use in solving the pipeline problem in
Section~\ref{sec:pipeline}.  Section~\ref{sec:hierarchical} discusses
the relationship between HA*LD and other hierarchical search methods.
In Sections \ref{sec:experiments} and \ref{sec:curves} we
present some experimental results.  We conclude in
Section~\ref{sec:conclusion}.

\section{Lightest Derivation Problems}
\label{sec:LDP}

Let $\Sigma$ be a set of statements and $R$ be a set of inference
rules of the following form,

\onerule{\irule{}
{\ant{A_1=w_1}
 \hspace{.5cm}\ant{\vdots}
 \ant{A_n=w_n}}
{}
{\con{C = g(w_1,\ldots,w_n)}}}

Here the antecedents $A_i$ and the conclusion $C$ are statements in
$\Sigma$, the weights $w_i$ are non-negative real valued variables and
$g$ is a non-negative real valued weight function.  For a rule with no
antecedents the function $g$ is simply a non-negative real value.
Throughout the paper we also use $A_1,\ldots,A_n \rightarrow_g C$ to
denote an inference rule of this type.

A \emph{derivation} of $C$ is a finite tree rooted at a rule
$A_1,\ldots,A_n \rightarrow_g C$ with $n$ children, where the $i$-th
child is a derivation of $A_i$.  The leaves of this tree are rules
with no antecedents.  Every derivation has a $\emph{weight}$ that is
the value obtained by recursive application of the functions $g$ along
the derivation tree.  Figure~\ref{fig:deriv} illustrates a derivation
tree.

\begin{figure}
\centering
\includegraphics[width=4.2in]{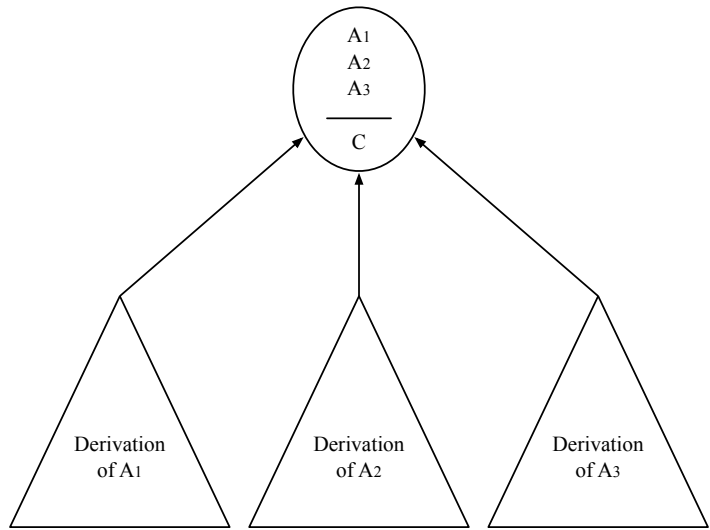}
\caption{A derivation of $C$ is a tree of rules rooted at a rule $r$
with conclusion $C$.  The children of the root are derivations of
the antecedents in $r$.  The leafs of the tree are rules with no antecedents.}
\label{fig:deriv}
\end{figure}

Intuitively a rule $A_1,\ldots,A_n \rightarrow_g C$ says that if we
can derive the antecedents $A_i$ with weights $w_i$ then we can derive
the conclusion $C$ with weight $g(w_1,\ldots,w_n)$.  The problem we
are interested in is to compute a lightest derivation of a special
goal statement.

All of the algorithms discussed in this paper assume that the weight
functions $g$ associated with a lightest derivation problem are
non-decreasing in each variable.  This is a fundamental property
ensuring that lightest derivations have an optimal substructure
property.  In this case lightest derivations can be constructed
from other lightest derivations.

To facilitate the runtime analysis of algorithms we assume that every
rule has a small number of antecedents.  We use $N$ to denote the
number of statements in a lightest derivation problem, while $M$
denotes the number of rules.  For most of the problems we are
interested in $N$ and $M$ are very large but the problem can be
implicitly defined in a compact way, by using a small number of rules
with variables as in the examples below.  We also assume that $N \le
M$ since statements that are not in the conclusion of some rule are
clearly not derivable and can be ignored.

\subsection{Dynamic Programming}

We say that a set of rules is \emph{acyclic} if there is an ordering
$O$ of the statements in $\Sigma$ such that for any rule with
conclusion $C$ the antecedents are statements that come before $C$ in
the ordering.  Dynamic programming can be used to solve a lightest
derivation problem if the functions $g$ in each rule are
non-decreasing and the set of rules is acyclic.  In this case lightest
derivations can be computed sequentially in terms of an acyclic
ordering $O$.  At the $i$-th step a lightest derivation of the $i$-th
statement is obtained by minimizing over all rules that can be used to
derive that statement.  This method takes $O(M)$ time to compute a
lightest derivation for each statement in $\Sigma$.  

We note that for cyclic rules it is sometimes possible to compute
lightest derivations by taking multiple passes over the statements.
We also note that some authors would refer to Dijkstra's algorithm
(and KLD) as a dynamic programming method.  In this paper we only use
the term when referring to algorithms that compute lightest derivations
in a fixed order that is independent of the solutions computed along
the way (this includes recursive implementations that use memoization).

\subsection{Examples}

Rules for computing shortest paths from a single source in a weighted
graph are shown in Figure~\ref{fig:path}.  We assume that we are given
a weighted graph $G=(V,E)$, where $w_{xy}$ is a non-negative weight
for each edge $(x,y) \in E$ and $s$ is a distinguished start node.
The first rule states that there is a path of weight zero to the start
node $s$.  The second set of rules state that if there is a path to a
node $x$ we can extend that path with an edge from $x$ to $y$ to
obtain an appropriately weighted path to a node $y$.  There is a rule
of this type for each edge in the graph.  A lightest derivation of
$\path{x}$ corresponds to shortest path from $s$ to $x$.  Note that
for general graphs these rules can be cyclic.  Figure~\ref{fig:paths}
illustrates a graph and two different derivations of $\path{b}$ using
the rules just described.  These corresponds to two different paths
from $s$ to $b$.

\begin{figure}
(1) 

\parbox{.2cm}{\hspace{.2cm}}
\parbox{2cm}{\irule{}{}{}
{\con{\path{s} = 0}}}

\vspace{.2cm}
(2) for each $(x,y) \in E$, 

\parbox{.2cm}{\hspace{.2cm}}
\parbox{2cm}{\irule{}
{\ant{\path{x} = w}}
{}
{\con{\path{y} = w + w_{xy}}}}
\caption{Rules for computing shortest paths in a graph.}
\label{fig:path}
\end{figure}

\begin{figure}
\centering
\includegraphics[width=5in]{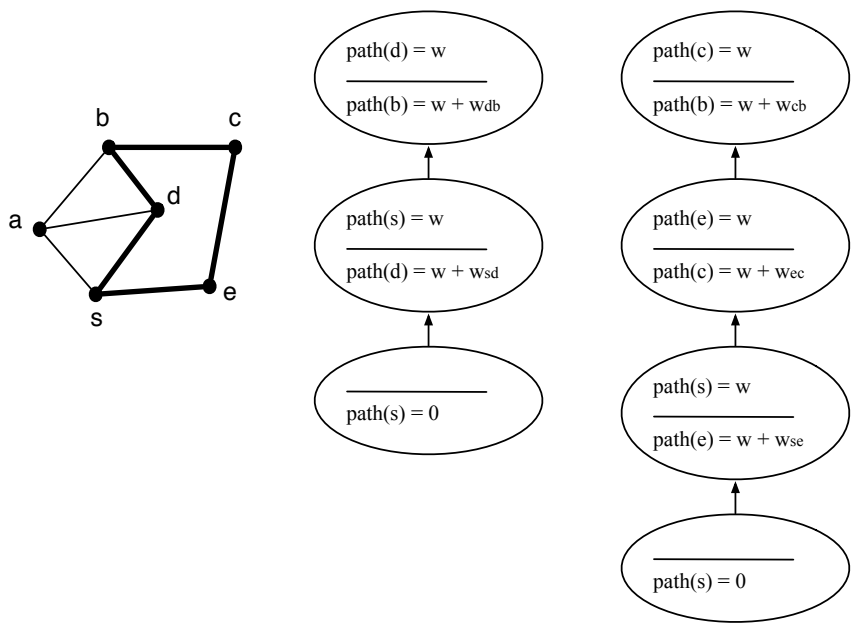}
\caption{A graph with two highlighted paths from $s$ to $b$ and
the corresponding derivations using rules from Figure~\ref{fig:path}.}
\label{fig:paths}
\end{figure}

Rules for chart parsing are shown in Figure~\ref{fig:parse}.  We
assume that we are given a weighted context free grammar in Chomsky
normal form \cite{Charniak}, i.e., a weighted set of productions
of the form $X \rightarrow s$ and $X \rightarrow YZ$ where $X$, $Y$
and $Z$ are nonterminal symbols and $s$ is a terminal symbol.  The
input string is given by a sequence of terminals $(s_1,\ldots,s_n)$.
The first set of rules state that if the grammar contains a production
$X \rightarrow s_i$ then there is a phrase of type $X$ generating the
$i$-th entry of the input with weight $w(X \rightarrow s_i)$.  The
second set of rules state that if the grammar contains a production
$X\rightarrow YZ$ and there is a phrase of type $Y$ from $i$ to $j$
and a phrase of type $Z$ from $j$ to $k$ then there is an,
appropriately weighted, phrase of type $X$ from $i$ to $k$.  Let $S$
be the start symbol of the grammar.  The goal of parsing is to find a
lightest derivation of $\phrase{S,1,n+1}$.  These rules are acyclic
because when phrases are composed together they form longer phrases.

\begin{figure}
(1) for each production $X \rightarrow s_i$,

\parbox{.2cm}{\hspace{1cm}}
\parbox{2cm}{\irule{}{}{}
{\con{\phrase{X,i,i+1} = w(X \rightarrow s_i)}}}

\vspace{.2cm}
(2) for each production $X \rightarrow YZ$ and $1 \le i < j < k \le n+1$,

\parbox{.2cm}{\hspace{1cm}}
\parbox{2cm}{\irule{}
{\ant{\phrase{Y,i,j} = w_1}
 \ant{\phrase{Z,j,k} = w_2}}
{}
{\con{\phrase{X,i,k} = w_1 + w_2 + w(X \rightarrow YZ)}}}
\caption{Rules for parsing with a context free grammar.}
\label{fig:parse}
\end{figure}

\subsection{AND/OR Graphs}
\label{sec:andor}

Lightest derivation problems are closely related to AND/OR graphs.
Let $\Sigma$ and $R$ be a set of statements and rules defining a
lightest derivation problem.  To convert the problem to an AND/OR
graph representation we can build a graph with a disjunction node for
each statement in $\Sigma$ and a conjunction node for each rule in
$R$.  There is an edge from each statement to each rule deriving that
statement, and an edge from each rule to its antecedents.  The leaves
of the AND/OR graph are rules with no antecedents.  Now derivations of
a statement using rules in $R$ can be represented by solutions rooted
at that statement in the corresponding AND/OR graph.  Conversely, it
is also possible to represent any AND/OR graph search problem as a
lightest derivation problem.  In this case we can view each node 
in the graph as a statement in $\Sigma$ and build an appropriate 
set of rules $R$.

\section{Knuth's Lightest Derivation}
\label{sec:KLD}

\citeA{Knuth77} described a generalization of Dijkstra's
shortest paths algorithm that we call Knuth's lightest derivation
(KLD).  Knuth's algorithm can be used to solve a large class of
lightest derivation problems.  The algorithm allows the rules to
be cyclic but requires that the weight functions associated with each
rule be non-decreasing and superior.  Specifically we require the
following two properties on the weight function $g$ in each rule,
\begin{eqnarray*}
\mbox{{\bf non-decreasing:}} && \mbox{if}\; w'_i \geq
w_i\;\mbox{then}\;g(w_1,\ldots,w'_i,\ldots,w_n) \geq
g(w_1,\ldots,w_i,\ldots,w_n) \\ \mbox{{\bf superior:}} &&
g(w_1,\ldots,w_n) \geq w_i
\end{eqnarray*}
For example,
$$\begin{array}{l}
g(x_1,\ldots,x_n) = x_1 + \cdots + x_n \\ 
g(x_1,\ldots,x_n) = \max(x_1,\ldots,x_n) 
\end{array}$$
are both non-decreasing and superior functions.

Knuth's algorithm computes lightest derivations in non-decreasing
weight order.  Since we are only interested in a lightest derivation
of a special goal statement we can often stop the algorithm before
computing the lightest derivation of every statement.

A \emph{weight assignment} is an expression of the form $(B=w)$ where
$B$ is a statement in $\Sigma$ and $w$ is a non-negative real value.
We say that the weight assignment $(B=w)$ is derivable if there is a
derivation of $B$ with weight $w$.  For any set of rules $R$,
statement $B$, and weight $w$ we write $R \vdash (B=w)$ if the rules
in $R$ can be used to derive $(B=w)$.  Let $\ell(B,R)$ be the
infimum of the set of weights derivable for $B$,
$$\ell(B,R) = \inf\{w:\;R\vdash (B=w)\}.$$ Given a set of rules $R$
and a statement $\goal \in \Sigma$ we are interested in computing
a derivation of $\goal$ with weight $\ell(\goal,R)$.

We define a bottom-up logic programming language in which we can
easily express the algorithms we wish to discuss throughout the rest
of the paper.  Each algorithm is defined by a set of rules with
priorities.  We encode the priority of a rule by writing it along the
line separating the antecedents and the conclusion as follows,

\onerule{\irule{}
{\ant{A_1=w_1}
 \hspace{.5cm}\ant{\vdots}
 \ant{A_n=w_n}}
{p(w_1,\ldots,w_n)}
{\con{C = g(w_1,\ldots,w_n)}}}

We call a rule of this form a \emph{prioritized rule}.  The execution
of a set of prioritized rules $P$ is defined by the procedure in
Figure~\ref{fig:algo}.  The procedure keeps track of a set ${\cal S}$
and a priority queue ${\cal Q}$ of weight assignments of the form
$(B=w)$.  Initially ${\cal S}$ is empty and ${\cal Q}$ contains weight
assignments defined by rules with no antecedents at the priorities
given by those rules.  We iteratively remove the lowest priority
assignment $(B=w)$ from ${\cal Q}$.  If $B$ already has an assigned
weight in ${\cal S}$ then the new assignment is ignored.  Otherwise we
add the new assignment to ${\cal S}$ and ``expand it'' --- every
assignment derivable from $(B=w)$ and other assignments already in
${\cal S}$ using some rule in $P$ is added to ${\cal Q}$ at the
priority specified by the rule.  The procedure stops when the queue is
empty.

\begin{figure}
\begin{algorithm}{Run}[P]{}
${\cal S} \qlet \emptyset$ \\
Initialize ${\cal Q}$ with assignments defined by rules
with no antecedents at their priorities \\
\qwhile ${\cal Q}$ is not empty \\
Remove the lowest priority element $(B=w)$ from ${\cal Q}$ \\
\qif $B$ has no assigned weight in ${\cal S}$ \\
${\cal S} \qlet {\cal S} \cup \{(B=w)\}$ \\
Insert assignments derivable from $(B=w)$ and other
assignments in ${\cal S}$ using some rule in $P$ into ${\cal Q}$
at the priority specified by the rule
\qend \qend \\
\qreturn ${\cal S}$ 
\end{algorithm}
\caption{Running a set of prioritized rules.}
\label{fig:algo}
\end{figure}

The result of executing a set of prioritized rules is a set of weight
assignments.  Moreover, the procedure can implicitly keep track of
derivations by remembering which assignments were used to derive an
item that is inserted in the queue.

\begin{lemma}
The execution of a finite set of prioritized rules $P$ derives every
statement that is derivable with rules in $P$.
\end{lemma}

\begin{proof}
Each rule causes at most one item to be inserted in the queue.  Thus
eventually ${\cal Q}$ is empty and the algorithm terminates.  When
${\cal Q}$ is empty every statement derivable by a single rule using
antecedents with weight in ${\cal S}$ already has a weight in ${\cal
S}$.  This implies that every derivable statement has a weight in
${\cal S}$.
\end{proof}

Now we are ready to define Knuth's lightest derivation algorithm.  The
algorithm is easily described in terms of prioritized rules.

\begin{definition}[Knuth's lightest derivation]
Let $R$ be a finite set of non-decreasing and superior rules.  Define
a set of prioritized rules ${\cal K}(R)$ by setting the priority of
each rule in $R$ to be the weight of the conclusion.  KLD is given by
the execution of ${\cal K}(R)$.
\end{definition}

We can show that while running ${\cal K}(R)$, if $(B = w)$ is added to
${\cal S}$ then $w=\ell(B,R)$.  This means that all assignments in
${\cal S}$ represent lightest derivations.  We can also show that
assignments are inserted into ${\cal S}$ in non-decreasing weight
order.  If we stop the algorithm as soon as we insert a weight
assignment for $\goal$ into ${\cal S}$ we will expand all statements
$B$ such that $\ell(B,R) < \ell(\goal,R)$ and some statements $B$ such
that $\ell(B,R) = \ell(\goal,R)$.  These properties follow from a more
general result described in the next section.

\subsection{Implementation}

The algorithm in Figure~\ref{fig:algo} can be implemented to run in
$O(M \log N)$ time, where $N$ and $M$ refer to the size of the problem
defined by the prioritized rules $P$.  

In practice the set of prioritized rules $P$ is often specified
implicitly, in terms of a small number of rules with variables.  In
this case the problem of executing $P$ is closely related to the work on logical
algorithms described by \citeA{McAllester02}.  

The main difficulty in devising an efficient implementation of the
procedure in Figure~\ref{fig:algo} is in step 7.  In that step we need
to find weight assignments in ${\cal S}$ that can be combined with
$(B=w)$ to derive new weight assignments.  The logical algorithms work
shows how a set of inference rules with variables can be transformed
into a new set of rules, such that every rule has at most two
antecedents and is in a particularly simple form.  Moreover, this
transformation does not increase the number of rules too much.  Once
the rules are transformed their execution can be implemented
efficiently using a hashtable to represent ${\cal S}$, a heap to
represent ${\cal Q}$ and indexing tables that allow us to perform step
7 quickly.

Consider the second set of rules for parsing in
Figure~\ref{fig:parse}.  These can be represented by a single rule
with variables.  Moreover the rule has two antecedents.  When
executing the parsing rules we keep track of a table mapping a value
for $j$ to statements $\phrase{Y,i,j}$ that have a weight in ${\cal
S}$.  Using this table we can quickly find statements that have a
weight in ${\cal S}$ and can be combined with a statement of the form
$\phrase{Z,j,k}$.  Similarly we keep track of a table mapping a value
for $j$ to statements $\phrase{Z,j,k}$ that have a weight in ${\cal
S}$.  The second table lets us quickly find statements that can be
combined with a statement of the form $\phrase{Y,i,j}$.  We refer the
reader to \cite{McAllester02} for more details.

\section{A* Lightest Derivation}
\label{sec:ALD}

Our A* lightest derivation algorithm (A*LD) is a generalization of A*
search to lightest derivation problems that subsumes A* parsing.  The
algorithm is similar to KLD but it can use a heuristic function to
speed up computation.  Consider a lightest derivation problem with
rules $R$ and goal statement $\goal$.  Knuth's algorithm will expand
any statement $B$ such that $\ell(B,R) < \ell(\goal,R)$.  By using a
heuristic function A*LD can avoid expanding statements that have light
derivations but are not part of a light derivation of $\goal$.

Let $R$ be a set of rules with statements in $\Sigma$, and $h$ be a
heuristic function assigning a weight to each statement.  Here $h(B)$ is
an estimate of the additional weight required to derive $\goal$ using a
derivation of $B$.  We note that in the case of a shortest path
problem this weight is exactly the distance from a node to the goal.
The value $\ell(B,R) + h(B)$ provides a \emph{figure of merit} for
each statement $B$.  The A* lightest derivation algorithm expands
statements in order of their figure of merit.

We say that a heuristic function is \emph{monotone} if for every rule
$A_1,\ldots,A_n \rightarrow_g C$ in $R$ and derivable weight
assignments $(A_i = w_i)$ we have,
\begin{equation}
w_i + h(A_i) \le g(w_1,\ldots,w_n) + h(C).
\label{eqn:monotone}
\end{equation}
This definition agrees with the standard notion of a monotone
heuristic function for rules that come from a shortest path problem.
We can show that if $h$ is monotone and $h(\goal) = 0$ then $h$ is
admissible under an appropriate notion of admissibility.  For the
correctness of A*LD, however, it is only required that $h$ be monotone
and that $h(\goal)$ be finite.  In this case monotonicity implies that
the heuristic value of every statement $C$ that appears in a
derivation of $\goal$ is finite.  Below we assume that $h(C)$ is
finite for every statement.  If $h(C)$ is not finite we can ignore $C$
and every rule that derives $C$.

\begin{definition}[A* lightest derivation]
Let $R$ be a finite set of non-decreasing rules and $h$ be a monotone
heuristic function for $R$.  Define a set of prioritized rules ${\cal
A}(R)$ by setting the priority of each rule in $R$ to be the weight of
the conclusion plus the heuristic value, $g(w_1,\ldots,w_n) + h(C)$.
A*LD is given by the execution of ${\cal A}(R)$.
\end{definition}

Now we show that the execution of ${\cal A}(R)$ correctly computes
lightest derivations and that it expands statements in order of their
figure of merit values.  

\begin{theorem}
During the execution of ${\cal A}(R)$, if $(B = w) \in {\cal S}$ then
$w=\ell(B,R)$.
\end{theorem}
\begin{proof}
The proof is by induction on the size of ${\cal S}$.  The statement is
trivial when ${\cal S} = \emptyset$.  Suppose the statement was true
right before the algorithm removed $(B = w_b)$ from ${\cal Q}$ and
added it to ${\cal S}$.  The fact that $(B = w_b) \in {\cal Q}$
implies that the weight assignment is derivable and thus $w_b \ge
\ell(B,R)$.

Suppose $T$ is a derivation of $B$ with weight $w_b' < w_b$.  Consider
the moment right before the algorithm removed $(B = w_b)$ from ${\cal
Q}$ and added it to ${\cal S}$.  Let $A_1,\ldots,A_n \rightarrow_g C$
be a rule in $T$ such that the antecedents $A_i$ have a weight in
${\cal S}$ while the conclusion $C$ does not.  Let $w_c =
g(\ell(A_1,R),\ldots,\ell(A_n,R))$.  By the induction hypothesis the
weight of $A_i$ in ${\cal S}$ is $\ell(A_i,R)$.  Thus $(C = w_c) \in
{\cal Q}$ at priority $w_c + h(C)$.  Let $w_c'$ be the weight that $T$
assigns to $C$.  Since $g$ is non-decreasing we know $w_c \le w_c'$.
Since $h$ is monotone $w_c' + h(C) \le w_b' + h(B)$.  This follows by
using the monotonicity condition along the path from $C$ to $B$ in
$T$.  Now note that $w_c + h(C) < w_b + h(B)$ which in turn implies
that $(B = w_b)$ is not the weight assignment in ${\cal Q}$ with
minimum priority.
\end{proof}

\begin{theorem}
During the execution of ${\cal A}(R)$ statements are expanded in order
of the figure of merit value $\ell(B,R) + h(B)$. 
\end{theorem}
\begin{proof}
First we show that the minimum priority of ${\cal Q}$ does not
decrease throughout the execution of the algorithm.  Suppose $(B = w)$
is an element in ${\cal Q}$ with minimum priority.  Removing $(B = w)$
from ${\cal Q}$ does not decrease the minimum priority.  Now suppose
we add $(B = w)$ to ${\cal S}$ and insert assignments derivable from
$(B=w)$ into ${\cal Q}$.  Since $h$ is monotone the priority of every
assignment derivable from $(B=w)$ is at least the priority of $(B=w)$.

A weight assignment $(B=w)$ is expanded when it is removed from ${\cal
Q}$ and added to ${\cal S}$.  By the last theorem $w = \ell(B,R)$ and
by the definition of ${\cal A}(R)$ this weight assignment was queued
at priority $\ell(B,R) + h(B)$.  Since we removed $(B=w)$ from ${\cal
Q}$ this must be the minimum priority in the queue.  The minimum
priority does not decrease over time so we must expand statements in
order of their figure of merit value.
\end{proof}

If we have accurate heuristic functions A*LD can be much more
efficient than KLD.  Consider a situation
where we have a perfect heuristic function.  That is, suppose $h(B)$
is exactly the additional weight required to derive $\goal$ using a
derivation of $B$.  Now the figure of merit $\ell(B,R) + h(B)$ equals
the weight of a lightest derivation of $\goal$ that uses $B$.  In this
case A*LD will derive $\goal$ before expanding any statements that
are not part of a lightest derivation of $\goal$.

The correctness KLD follows from the correctness of A*LD.  For a set
of non-decreasing and superior rules we can consider the trivial
heuristic function $h(B) = 0$.  The fact that the rules are superior
imply that this heuristic is monotone.  The theorems above imply that
Knuth's algorithm correctly computes lightest derivations and expands
statements in order of their lightest derivable weights.

\section{Heuristics Derived from Abstractions}
\label{sec:abs}

Here we consider the case of additive rules --- rules where the weight
of the conclusion is the sum of the weights of the antecedents plus a
non-negative value $v$ called the weight of the rule.  We denote such
a rule by $A_1,\ldots,A_n \rightarrow_v C$.  The weight of a
derivation using additive rules is the sum of the weights of the rules
that appear in the derivation tree.

A \emph{context} for a statement $B$ is a finite tree of rules such
that if we add a derivation of $B$ to the tree we get a derivation of
$\goal$.  Intuitively a context for $B$ is a derivation of $\goal$
with a ``hole'' that can be filled in by a derivation of $B$ (see
Figure~\ref{fig:context}).

\begin{figure}
\centering
\includegraphics[width=5in]{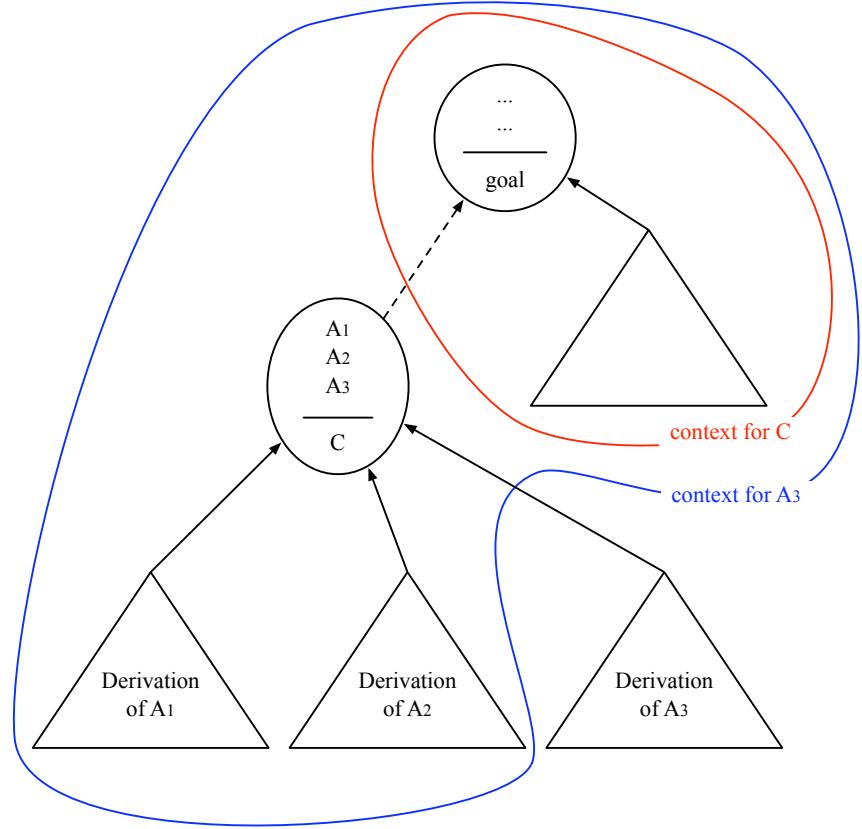}
\caption{A derivation of $\goal$ defines contexts for the statements
that appear in the derivation tree.  Note how a context for $C$
together with a rule $A_1,A_2,A_3 \rightarrow C$ and derivations of
$A_1$ and $A_2$ define a context for $A_3$.}
\label{fig:context}
\end{figure}

For additive rules, each context has a weight that is the sum of
weights of the rules in it.  Let $R$ be a set of additive rules with
statements in $\Sigma$.  For $B \in \Sigma$ we define
$\ell(\context(B),R)$ to be the weight of a lightest context for $B$.
The value $\ell(B,R) + \ell(\context(B),R)$ is the weight of a
lightest derivation of $\goal$ that uses $B$.

Contexts can be derived using rules in $R$ together with context rules
$c(R)$ defined as follows.  First, $goal$ has an empty context with
weight zero.  This is captured by a rule with no antecedents 
$\rightarrow_0 \context(\goal)$.  For each rule
$A_1,\ldots,A_n \rightarrow_v C$ in $R$ we put $n$ rules in $c(R)$.
These rules capture the notion that a context for $C$ and derivations
of $A_j$ for $j \neq i$ define a context for $A_i$,
$$\context(C),A_1,\ldots,A_{i-1},A_{i+1},\ldots,A_n \rightarrow_v
\context(A_i).$$ 
Figure~\ref{fig:context} illustrates how a context for $C$ together
with derivations of $A_1$ and $A_2$ and a rule $A_1,A_2,A_3
\rightarrow C$ define a context for $A_3$.

We say that a heuristic function $h$ is \emph{admissible} if $h(B) \le
\ell(\context(B),R)$.  Admissible heuristic functions never
over-estimate the weight of deriving $\goal$ using a derivation of
a particular statement.  The heuristic function is $\emph{perfect}$ if
$h(B) = \ell(\context(B),R)$.  Now we show how to obtain admissible
and monotone heuristic functions from abstractions.

\subsection{Abstractions}
\label{sec:abstractions}

Let $(\Sigma,R)$ be a lightest derivation problem with statements
$\Sigma$ and rules $R$.  An \emph{abstraction} of $(\Sigma,R)$ is
given by a problem $(\Sigma',R')$ and a map
\map{\abs}{\Sigma}{\Sigma'}, such that for every rule
$A_1,\ldots,A_n \rightarrow_v C$ in $R$ there is a rule
$abs(A_1),\ldots,abs(A_n) \rightarrow_{v'} \abs(C)$ in $R'$ with
$v' \le v$.  Below we show how an abstraction can be used to
define a monotone and admissible heuristic function for the
original problem.

We usually think of $\abs$ as defining a coarsening of $\Sigma$ by
mapping several statements into the same abstract statement.  For
example, for a parser $\abs$ might map a lexicalized nonterminal
$NP_{\mathrm{house}}$ to the nonlexicalized nonterminal $NP$.  In this
case the abstraction defines a smaller problem on the abstract
statements.  Abstractions can often be defined in a mechanical way by
starting with a map $\abs$ from $\Sigma$ into some set of abstract
statements $\Sigma'$.  We can then ``project'' the rules in $R$ from
$\Sigma$ into $\Sigma'$ using $\abs$ to get a set of abstract rules.
Typically several rules in $R$ will map to the same abstract rule.  We
only need to keep one copy of each abstract rule, with a weight that
is a lower bound on the weight of the concrete rules mapping into it.

Every derivation in $(\Sigma,R)$ maps to an abstract derivation so we
have $\ell(\abs(C),R') \le \ell(C,R)$.  If we let the goal of the
abstract problem be $\abs(\goal)$ then every context in $(\Sigma,R)$
maps to an abstract context and we see that
$\ell(\context(\abs(C)),R') \le \ell(\context(C),R)$.  This means that
lightest abstract context weights form an admissible heuristic
function,
$$h(C) = \ell(\context(\abs(C)),R').$$ 
Now we show that this heuristic function is also monotone.  

Consider a rule $A_1,\ldots,A_n \rightarrow_v C$ in $R$ and let $(A_i
= w_i)$ be weight assignments derivable using $R$.  In this case there
is a rule $\abs(A_1),\ldots,\abs(A_n) \rightarrow_{v'} \abs(C)$ in $R'$
where $v' \le v$ and $(abs(A_i) = w_i')$ is derivable using $R'$ where
$w_i' \le w_i$.  By definition of contexts (in the abstract problem)
we have,
$$\ell(\context(\abs(A_i)),R') \le v' + \sum_{j \neq i} w_j' +
\ell(\context(\abs(C)),R').$$ 
Since $v' \le v$ and $w_j' \le w_j$ we have,
$$\ell(\context(\abs(A_i)),R') \le v + \sum_{j \neq i} w_j +
\ell(\context(\abs(C)),R').$$ Plugging in the heuristic function $h$
from above and adding $w_i$ to both sides,
$$w_i + h(A_i) \le v + \sum_{j} w_j + h(C),$$ which is exactly the
monotonicity condition in equation~(\ref{eqn:monotone}) for an
additive rule.

If the abstract problem defined by $(\Sigma',R')$ is relatively small
we can efficiently compute lightest context weights for every
statement in $\Sigma'$ using dynamic programming or KLD.  We can store
these weights in a ``pattern database'' (a lookup table) to serve as a
heuristic function for solving the concrete problem using A*LD.  This
heuristic may be able to stop A*LD from exploring a lot of
non-promising structures.  This is exactly the approach that was used
by \citeA{Culberson98} and \citeA{Korf97} for solving very large
search problems.  The results in this section show that pattern
databases can be used in the more general setting of lightest
derivations problems.  The experiments in Section~\ref{sec:curves}
demonstrate the technique in a specific application.

\section{Hierarchical A* Lightest Derivation}
\label{sec:HALD}



The main disadvantage of using pattern databases is that we have to
precompute context weights for every abstract statement.  This can
often take a lot of time and space.  Here we define a hierarchical
algorithm, HA*LD, that searches for lightest derivations and contexts
in an entire abstraction hierarchy simultaneously.  This algorithm can
often solve the most concrete problem without fully computing context
weights at any level of abstraction.

At each level of abstraction the behavior HA*LD is similar to the
behavior of A*LD when using an abstraction-derived heuristic function.
The hierarchical algorithm queues derivations of a statement $C$ at a
priority that depends on a lightest abstract context for $C$.  But now
abstract contexts are not computed in advance.  Instead, abstract
contexts are computed at the same time we are computing derivations.
Until we have an abstract context for $C$, derivations of $C$ are
``stalled''.  This is captured by the addition of $\context(\abs(C))$
as an antecedent to each rule that derives $C$.

We define an abstraction hierarchy with $m$ levels to be a sequence of
lightest derivation problems with additive rules $(\Sigma_k,R_k)$ for
$0 \le k \le m-1$ with a single abstraction function $\abs$.  For $0
\le k < m-1$ the abstraction function maps $\Sigma_k$ onto
$\Sigma_{k+1}$.  We require that $(\Sigma_{k+1},R_{k+1})$ be an
abstraction of $(\Sigma_k, R_k)$ as defined in the previous section:
if $A_1,\ldots,A_n \rightarrow_v C$ is in $R_k$ then there exists a
rule $\abs(A_1),\ldots,\abs(A_n) \rightarrow_{v'}\abs(C)$ in $R_{k+1}$
with $v' \leq v$.  The hierarchical algorithm computes lightest
derivations of statements in $\Sigma_k$ using contexts from
$\Sigma_{k+1}$ to define heuristic values.  We extend $\abs$ so that
it maps $\Sigma_{m-1}$ to a most abstract set of statements $\Sigma_m$
containing a single element $\bot$.  Since $\abs$ is onto we have
$|\Sigma_k| \geq |\Sigma_{k+1}|$.  That is, the number of statements
decrease as we go up the abstraction hierarchy.  We denote by $\abs_k$
the abstraction function from $\Sigma_0$ to $\Sigma_k$ obtained by
composing $\abs$ with itself $k$ times.

We are interested in computing a lightest derivation of a goal
statement $\goal \in \Sigma_0$.  Let $\goal_k = \abs_k(\goal)$ be the
goal at each level of abstraction.  The hierarchical algorithm is
defined by the set of prioritized rules ${\cal H}$ in
Figure~\ref{fig:all}.  Rules labeled UP compute derivations of
statements at one level of abstraction using context weights from the
level above to define priorities.  Rules labeled BASE and DOWN compute
contexts in one level of abstraction using derivation weights at the
same level to define priorities.  The rules labeled START1 and START2
start the inference by handling the most abstract level.

\begin{figure}
\parbox[t]{2cm}{START1: }
\parbox[t]{2.0in}{\vspace{-2.2em}
\irule{}
{}
{0}
{\ant{\bot = 0}}}

\parbox[t]{2cm}{START2: }
\parbox[t]{2.0in}{\vspace{-2.2em}
\irule{}
{}
{0}
{\ant{\context(\bot) = 0}}}

\parbox[t]{2cm}{BASE: }
\parbox[t]{2.0in}{\vspace{-2.2em}
\irule{}
{\ant{\goal_k = w}}
{w}
{\ant{\context(\goal_k) = 0}}}

\parbox[t]{2cm}{UP: }
\parbox[t]{2.0in}{\vspace{-2.2em}
\irule{}
{\ant{\context(\abs(C))=w_c}
\ant{A_1=w_1}
\ant{\hspace{1ex}\vdots}
\ant{A_n=w_n}}
{v + w_1+\cdots+w_n+w_c}
{\ant{C=v+w_1+\cdots+w_n}}}

\parbox[t]{2cm}{DOWN: }
\parbox[t]{2.0in}{\vspace{-2.2em}
\irule{}
{\ant{\context(C) = w_c}
\ant{A_1=w_1}
\ant{\hspace{1ex}\vdots}
\ant{A_n=w_n}}
{v+w_c+w_1+\cdots+w_n}
{\con{\context(A_i) = v+w_c+w_1+\cdots+w_n-w_i}}}

\caption{Prioritized rules ${\cal H}$ defining HA*LD.  BASE rules are
defined for $0 \le k \le m-1$.  UP and DOWN rules are defined for
each rule $A_1,\ldots,A_n \rightarrow_v C \in R_k$ with $0 \le k \le
m-1$.}
\label{fig:all}
\end{figure}

The execution of ${\cal H}$ starts by computing a derivation and
context for $\bot$ with START1 and START2.  It continues by deriving
statements in $\Sigma_{m-1}$ using UP rules.  Once the lightest
derivation of $\goal_{m-1}$ is found the algorithm derives a context
for $\goal_{m-1}$ with a BASE rule and starts computing contexts for
other statements in $\Sigma_{m-1}$ using DOWN rules.  In general HA*LD
interleaves the computation of derivations and contexts at each level
of abstraction since the execution of ${\cal H}$ uses a single
priority queue.  

Note that no computation happens at a given level of abstraction until
a lightest derivation of the goal has been found at the level above.
This means that the structure of the abstraction hierarchy can be
defined dynamically.  For example, as in the CFDP algorithm, we could
define the set of statements at each level of abstraction by refining
the statements that appear in a lightest derivation of the goal at the
level above.  Here we assume a static abstraction hierarchy.

For each statement $C \in \Sigma_k$ with $0 \le k \le m-1$ we use
$\ell(C)$ to denote the weight of a lightest derivation for $C$ using
$R_k$ while $\ell(\context(C))$ denotes the weight of a lightest
context for $C$ using $R_k$.  For the most abstract level we define
$\ell(\bot) = \ell(\context(\bot)) = 0$.

Below we show that HA*LD correctly computes lightest derivations and
lightest contexts at every level of abstraction.  Moreover, the order
in which derivations and contexts are expanded is controlled by a
heuristic function defined as follows.  For $C \in \Sigma_k$ with $0
\le k \le m-1$ define a heuristic value for $C$ using contexts at the
level above and a heuristic value for $\context(C)$ using derivations
at the same level,
$$\begin{array}{l}
h(C) = \ell(\context(\abs(C))), \\
h(\context(C)) = \ell(C).
\end{array}$$
For the most abstract level we define $h(\bot) = h(\context(\bot)) =
0$.  Let a generalized statement $\Phi$ be either an element of $\Sigma_k$ for $0 \le k \le m$ or
an expression of the form $\context(C)$ for $C \in \Sigma_k$.  We
define an intrinsic priority for $\Phi$ as follows,
$$p(\Phi) = \ell(\Phi) + h(\Phi).$$ 
For $C \in \Sigma_k$, we have that $p(\context(C))$ is the weight of
a lightest derivation of $\goal_k$ that uses $C$, while $p(C)$ is a lower
bound on this weight.  

The results from Sections \ref{sec:ALD} and \ref{sec:abs} cannot be
used directly to show the correctness of HA*LD.  This is because the
rules in Figure~\ref{fig:all} generate heuristic values at the same
time they generate derivations that depend on those heuristic values.
Intuitively we must show that during the execution of the prioritized
rules ${\cal H}$, each heuristic value is available at an appropriate
point in time.  The next lemma shows that the rules in ${\cal H}$
satisfy a monotonicity property with respect to the intrinsic
priority of generalized statements.  Theorem~\ref{thm:main}
proves the correctness of the hierarchical algorithm.

\begin{lemma}[Monotonicity]
For each rule $\Phi_1,\ldots,\Phi_m \rightarrow \Psi$ in the
hierarchical algorithm, if the weight of each antecedent $\Phi_i$ is
$\ell(\Phi_i)$ and the weight of the conclusion is $\alpha$ then
\begin{itemize}
\item[(a)] the priority of the rule is $\alpha+h(\Psi)$.
\item[(b)] $\alpha+h(\Psi) \ge p(\Phi_i)$.
\end{itemize}
\end{lemma}

\begin{proof}
For the rules START1 and START2 the result follows from the fact that
the rules have no antecedents and $h(\bot) = h(\context(\bot)) = 0$.

Consider a rule labeled BASE with $w=\ell(\goal_k)$.  To see (a) note
that $\alpha$ is always zero and the priority of the rule is $w =
h(\context(\goal_k))$.  For (b) we note that $p(\goal_k) =
\ell(\goal_k)$ which equals the priority of the rule.

Now consider a rule labeled UP with $w_c=\ell(\context(\abs(C)))$ and
$w_i = \ell(A_i)$ for all $i$.  For part (a) note how the priority of
the rule is $\alpha + w_c$ and $h(C) = w_c$.  For part (b)
consider the first antecedent of the rule.  We have
$h(\context(\abs(C))) = \ell(\abs(C)) \le \ell(C) \le \alpha$, and
$p(\context(\abs(C))) = w_c + h(\context(\abs(C))) \le \alpha + w_c$.
Now consider an antecedent $A_i$.  If $\abs(A_i) = \bot$ then $p(A_i)
= w_i \le \alpha + w_c$.  If $\abs(A_i) \neq \bot$ then we can show
that $h(A_i) = \ell(\context(\abs(A_i))) \le w_c + \alpha - w_i$.
This implies that $p(A_i) = w_i + h(A_i) \le \alpha + w_c$.

Finally consider a rule labeled DOWN with $w_c=\ell(\context(C))$ and
$w_j = \ell(A_j)$ for all $j$.  For part (a) note that the priority of
the rule is $\alpha + w_i$ and $h(\context(A_i)) = w_i$.  For part (b)
consider the first antecedent of the rule.  We have $h(\context(C)) =
\ell(C) \le v + \sum_j w_j$ and we see that $p(\context(C)) = w_c +
h(C) \le \alpha + w_i$.  Now consider an antecedent $A_j$.  If
$\abs(A_j) = \bot$ then $h(A_j) = 0$ and $p(A_j) = w_j \le \alpha +
w_i$.  If $\abs(A_j) \neq \bot$ we can show that $h(A_j) \le \alpha +
w_i - w_j$.  Hence $p(A_j) = w_j + h(A_j) \le \alpha + w_i$.
\end{proof}

\begin{theorem}
The execution of ${\cal H}$ maintains the following invariants.
\begin{enumerate}
\item If $(\Phi=w) \in {\cal S}$ then $w = \ell(\Phi)$.
\item If $(\Phi=w) \in {\cal Q}$ then it has priority $w+h(\Phi)$.
\item If $p(\Phi) < p({\cal Q})$ then $(\Phi=\ell(\Phi)) \in {\cal S}$
\end{enumerate}
Here $p({\cal Q})$ denotes the smallest priority in ${\cal Q}$.
\label{thm:main}
\end{theorem}

\begin{proof}
In the initial state of the algorithm ${\cal S}$ is empty and ${\cal
Q}$ contains only $(\bot = 0)$ and $(\context(\bot)=0)$ at priority
$0$.  For the initial state invariant 1 is true since ${\cal S}$ is
empty; invariant 2 follows from the definition of $h(\bot)$ and
$h(\context(\bot))$; and invariant 3 follows from the fact that
$p({\cal Q}) = 0$ and $p(\Phi) \ge 0$ for all $\Phi$.  Let ${\cal S}$
and ${\cal Q}$ denote the state of the algorithm immediately prior to
an iteration of the while loop in Figure~\ref{fig:algo} and suppose
the invariants are true.  Let ${\cal S}'$ and ${\cal Q}'$ denote the
state of the algorithm after the iteration.  

We will first prove invariant 1 for ${\cal S}'$.  Let $(\Phi=w)$ be
the element removed from ${\cal Q}$ in this iteration.  By the
soundness of the rules we have $w \geq \ell(\Phi)$. If $w =
\ell(\Phi)$ then clearly invariant 1 holds for ${\cal S}'$.  If $w >
\ell(\Phi)$ invariant 2 implies that $p({\cal Q}) = w + h(\Phi) >
\ell(\Phi)+h(\Phi)$ and by invariant 3 we know that ${\cal S}$
contains $(\Phi=\ell(\Phi))$.  In this case ${\cal S}' = {\cal S}$.

Invariant 2 for ${\cal Q}'$ follows from invariant 2 for ${\cal Q}$,
invariant 1 for ${\cal S}'$, and part (a) of the monotonicity lemma.

Finally, we consider invariant 3 for ${\cal S}'$ and ${\cal Q}'$.  The
proof is by reverse induction on the abstraction level of $\Phi$.  We
say that $\Phi$ has level $k$ if $\Phi \in \Sigma_k$ or $\Phi$ is of
the form $\context(C)$ with $C \in \Sigma_k$.  In the reverse
induction, the base case considers $\Phi$ at level
$m$.  Initially the algorithm inserts $(\bot = 0)$ and
$(\context(\bot) = 0)$ in the queue with priority $0$.  If $p({\cal Q}')
> 0$ then ${\cal S}'$ must contain $(\bot = 0)$ and $(\context(\bot) =
0)$.  Hence invariant 3 holds for ${\cal S}'$ and ${\cal Q}'$ with $\Phi$
at level $m$.

Now we assume that invariant 3 holds for ${\cal S}'$ and ${\cal Q}'$
with $\Phi$ at levels greater than $k$ and consider level $k$.  We
first consider statements $C \in \Sigma_k$.  Since the rules $R_k$
are additive, every statement $C$ derivable with $R_k$ has a lightest
derivation (a derivation with weight $\ell(C)$).  This follows
from the correctness of Knuth's algorithm.  Moreover, for additive
rules, subtrees of lightest derivations are also lightest derivations.
We show by structural induction that for any lightest derivation with
conclusion $C$ such that $p(C) < p({\cal Q}')$ we have $(C = \ell(C))
\in {\cal S}'$.  Consider a lightest derivation in $R_k$ with
conclusion $C$ such that $p(C) < p({\cal Q}')$.  The final rule in
this derivation $A_1,\ldots,A_n \rightarrow_v C$ corresponds to an UP
rule where we add an antecedent for $\context(\abs(C))$.  By part (b)
of the monotonicity lemma all the antecedents of this UP rule have
intrinsic priority less than $p({\cal Q}')$.  By the induction
hypothesis on lightest derivations we have $(A_i = \ell(A_i)) \in
{\cal S}'$.  Since invariant 3 holds for statements at levels greater
than $k$ we have $(\context(\abs(C)) = \ell(\context(\abs(C)))) \in
{\cal S}'$.  This implies that at some point the UP rule was used to
derive $(C=\ell(C))$ at priority $p(C)$.  But $p(C) < p({\cal Q}')$
and hence this item must have been removed from the queue.  Therefore
${\cal S}'$ must contain $(C=w)$ for some $w$ and, by invariant 1,
$w=\ell(C)$.

Now we consider $\Phi$ of the form $\context(C)$ with $C \in
\Sigma_k$.  As before we see that $c(R_k)$ is additive and thus every
statement derivable with $c(R_k)$ has a lightest derivation and
subtrees of lightest derivations are lightest derivations themselves.
We prove by structural induction that for any lightest derivation $T$
with conclusion $\context(C)$ such that $p(\context(C)) < p({\cal
Q}')$ we have $(\context(C) = \ell(\context(C))) \in {\cal S}'$.
Suppose the last rule of $T$ is of the form,
$$\context(C),A_1,\ldots,A_{i-1},A_{i+1},\ldots,A_n \rightarrow_v
\context(A_i).$$ This rule corresponds to a DOWN rule where we add an
antecedent for $A_i$.  By part (b) of the monotonicity lemma all the
antecedents of this DOWN rule have intrinsic priority less than
$p({\cal Q}')$.  By invariant 3 for statements in $\Sigma_k$ and by
the induction hypothesis on lightest derivations using $c(R_k)$, all
antecedents of the DOWN rule have their lightest weight in ${\cal
S}'$.  So at some point $(\context(A_i) = \ell(\context(A_i)))$ was
derived at priority $p(A_i)$.  Now $p(A_i) < p({\cal Q}')$ implies the
item was removed from the queue and, by invariant 1, we have
$(\context(A_i) = \ell(\context(A_i))) \in {\cal S}'$.

Now suppose the last (and only) rule in $T$ is $\rightarrow_0
\context(\goal_k)$.  This rule corresponds to a BASE rule where we add
$\goal_k$ as an antecedent.  Note that $p(\goal_k) = \ell(\goal_k) = 
p(\context(\goal_k))$ and hence $p(\goal_k) < p({\cal Q}')$.  By
invariant 3 for statements in $\Sigma_k$ we have $(\goal_k =
\ell(\goal_k))$ in ${\cal S}'$ and at some point the BASE rule was
used to queue $(\context(\goal_k) = \ell(\context(\goal_k)))$ at
priority $p(\context(\goal_k))$.  As in the previous cases
$p(\context(\goal_k)) < p({\cal Q}')$ implies $(\context(\goal_k) =
\ell(\context(\goal_k))) \in {\cal S}'$.
\end{proof}

The last theorem implies that generalized statements $\Phi$ are
expanded in order of their intrinsic priority.  Let $K$ be the number
of statements $C$ in the entire abstraction hierarchy with $p(C) \le
p(\goal) = \ell(\goal)$.  For every statement $C$ we have that $p(C)
\le p(\context(C))$.  We conclude that HA*LD expands at most $2K$
generalized statements before computing a lightest derivation of
$\goal$.

\subsection{Example}

Now we consider the execution of HA*LD in a specific example.  The
example illustrates how HA*LD interleaves the computation of
structures at different levels of abstraction.  Consider the following
abstraction hierarchy with 2 levels.
$$\Sigma_0 = \{X_1,\ldots,X_n,Y_1,\ldots,Y_n,Z_1,\ldots,Z_n,\goal_0\}, \;\;\;
\Sigma_1 = \{X,Y,Z,\goal_1\},$$
$$R_0 = \left\{
\begin{array}{l}
\rightarrow_i X_i, \\
\rightarrow_i Y_i, \\
X_i,Y_j \rightarrow_{i*j} \goal_0, \\
X_i,Y_i \rightarrow_5 Z_i, \\
Z_i \rightarrow_i \goal_0,
\end{array}
\right\}, \;\;\;
R_1 = \left\{
\begin{array}{l}
\rightarrow_1 X, \\
\rightarrow_1 Y, \\
X,Y \rightarrow_1 \goal_1, \\
X,Y \rightarrow_5 Z, \\
Z \rightarrow_1 \goal_1,
\end{array}
\right\},$$ $$ $$ with $\abs(X_i) = X$, $\abs(Y_i) = Y$, $\abs(Z_i) =
Z$ and $\abs(goal_0) = goal_1$.  

\begin{enumerate}
\item Initially ${\cal S} = \emptyset$ and ${\cal Q} = \{(\bot=0)
\text{ and } (\context(\bot)=0) \text{ at priority 0}\}$.

\item When $(\bot = 0)$ comes off the queue it gets put in ${\cal S}$
but nothing else happens.

\item When $(\context(\bot) = 0)$ comes off the queue it gets put in
${\cal S}$.  Now statements in $\Sigma_1$ have an abstract context
in ${\cal S}$.  This causes UP rules that come from rules in $R_1$
with no antecedents to ``fire'', putting $(X = 1)$ and $(Y = 1)$ in
${\cal Q}$ at priority 1.

\item When $(X = 1)$ and $(Y = 1)$ come off the queue they get put in
${\cal S}$, causing two UP rules to fire, putting $(\goal_1 = 3)$ at
priority 3 and $(Z = 7)$ at priority 7 in the queue.

\item We have,
$${\cal S} = \{(\bot = 0),\; (\context(\bot) = 0),\; (X = 1),\;
(Y = 1)\}$$ $${\cal Q} = \{(\goal_1 = 3) \text{ at priority 3},\; (Z =
7) \text{ at priority 7}\}$$

\item At this point $(\goal_1 = 3)$ comes off the queue and goes into
in ${\cal S}$.  A BASE rule fires putting $(\context(\goal_1) = 0)$ in
the queue at priority 3.

\item $(\context(\goal_1) = 0)$ comes off the queue.  This is the base
case for contexts in $\Sigma_1$.  Two DOWN rules use
$(\context(\goal_1) = 0)$, $(X = 1)$ and $(Y = 1)$ to put
$(\context(X) = 2)$ and $(\context(Y) = 2)$ in ${\cal Q}$ at priority
3.

\item $(\context(X) = 2)$ comes off the queue and gets put in ${\cal
S}$.  Now we have an abstract context for each $X_i \in
\Sigma_0$, so UP rules to put $(X_i = i)$ in ${\cal Q}$ at
priority $i+2$.

\item Now $(\context(Y) = 2)$ comes off the queue and goes into 
${\cal S}$.  As in the previous step UP rules put $(Y_i = i)$ in
${\cal Q}$ at priority $i+2$.

\item We have,
\begin{eqnarray*}
{\cal S} & = & \{(\bot = 0),\; (\context(\bot) = 0),\; (X = 1),\; (Y =
1),\; (\goal_1 = 3),\; \\ & & (\context(\goal_1) = 0),\; (\context(X) = 2),\;
(\context(Y) = 2)\}
\end{eqnarray*}
$${\cal Q} = \{(X_i = i) \text{ and } (Y_i = i) \text{ at
  priority $i+2$ for $1 \le i \le n$},\; (Z = 7) \text{ at priority 7}\}$$

\item Next both $(X_1 = 1)$ and $(Y_1 = 1)$ will come off the queue
and go into ${\cal S}$.  This causes an UP rule to put $(\goal_0 = 3)$
in the queue at priority 3.

\item $(\goal_0 = 3)$ comes off the queue and goes into ${\cal S}$.
The algorithm can stop now since we have a derivation of
the most concrete goal.
\end{enumerate}

Note how HA*LD terminates before fully computing abstract derivations
and contexts.  In particular $(Z=7)$ is in ${\cal Q}$ but $Z$ was
never expanded.  Moreover $\context(Z)$ is not even in the queue.  If
we keep running the algorithm it would eventually derive
$\context(Z)$, and that would allow the $Z_i$ to be derived.  

\section{The Perception Pipeline}
\label{sec:pipeline}

Figure~\ref{fig:pipeline} shows a hypothetical run of the hierarchical
algorithm for a processing pipeline of a vision system.  In this
system weighted statements about edges are used to derive weighted
statements about contours which provide input to later stages
ultimately resulting in statements about recognized objects.

It is well known that the subjective presence of edges at a particular
image location can depend on the context in which a given image patch
appears.  This can be interpreted in the perception pipeline by
stating that higher level processes --- those later in the pipeline
--- influence low-level interpretations.  This kind of influence
happens naturally in a lightest derivation problem.  For example, the
lightest derivation of a complete scene analysis might require the
presence of an edge that is not locally apparent.  By implementing the
whole system as a single lightest derivation problem we avoid the need to
make hard decisions between stages of the pipeline.

\begin{figure}
\centering \includegraphics[width=3.3in]{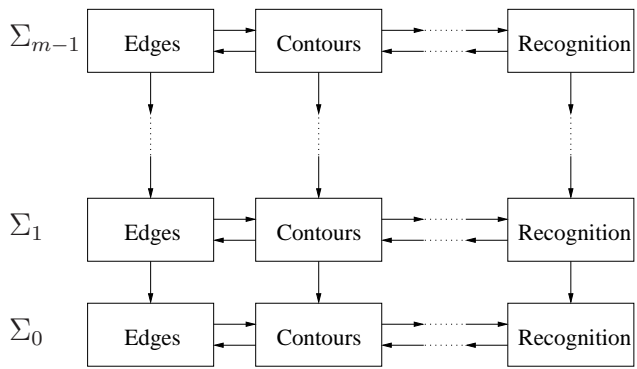}
\caption{A vision system with several levels of processing.  Forward
  arrows represent the normal flow of information from one stage of
  processing to the next.  Backward arrows represent the computation
  of contexts.  Downward arrows represent the influence of contexts.}
\label{fig:pipeline}
\end{figure}

The influence of late pipeline stages in guiding earlier stages is
pronounced if we use HA*LD to compute lightest derivations. In this
case the influence is apparent not only in the structure of the
optimal solution but also in the flow of information across different
stages of processing.  In HA*LD a complete interpretation derived at
one level of abstraction guides all processing stages at a more
concrete level.  Structures derived at late stages of the pipeline
guide earlier stages through abstract context weights.  This allows
the early processing stages to concentrate computational efforts in
constructing structures that will likely be part of the globally
optimal solution.

While we have emphasized the use of admissible heuristics, we note
that the A* architecture, including HA*LD, can also be used with
inadmissible heuristic functions (of course this would break our
optimality guarantees).  Inadmissible heuristics are important because
admissible heuristics tend to force the first few stages of a
processing pipeline to generate too many derivations.  As derivations
are composed their weights increase and this causes a large number of
derivations to be generated at the first few stages of processing
before the first derivation reaches the end of the pipeline.
Inadmissible heuristics can produce behavior similar to beam search
--- derivations generated in the first stage of the pipeline can flow
through the whole pipeline quickly.  A natural way to construct
inadmissible heuristics is to simply ``scale-up'' an admissible
heuristic such as the ones obtained from abstractions.  It is then
possible to construct a hierarchical algorithm where inadmissible
heuristics obtained from one level of abstraction are used to guide
search at the level below.

\section{Other Hierarchical Methods}
\label{sec:hierarchical}

In this section we compare HA*LD to other hierarchical search methods.

\subsection{Coarse-to-Fine Dynamic Programming}

HA*LD is related to the coarse-to-fine dynamic programming (CFDP)
method described by \citeA{Raphael01}.  To understand the relationship
consider the problem of finding the shortest path from $s$ to $t$ in a
trellis graph like the one shown in Figure~\ref{fig:cfdp}(a).  Here we
have $k$ columns of $n$ nodes and every node in one column is
connected to a constant number of nodes in the next column.  Standard
dynamic programming can be used to find the shortest path in $O(kn)$
time.  Both CFDP and HA*LD can often find the shortest path much
faster.  On the other hand the worst case behavior of these algorithms
is very different as we describe below, with CFDP taking
significantly more time than HA*LD.

The CFDP algorithm works by coarsening the graph, grouping nodes in
each column into a small number of supernodes as illustrated in
Figure~\ref{fig:cfdp}(b).  The weight of an edge between two
supernodes $A$ and $B$ is the minimum weight between nodes $a \in A$
and $b \in B$.  The algorithm starts by using dynamic programming to
find the shortest path $P$ from $s$ to $t$ in the coarse graph, this
is shown in bold in Figure~\ref{fig:cfdp}(b).  The supernodes along
$P$ are partitioned to define a finer graph as shown in
Figure~\ref{fig:cfdp}(c) and the procedure repeated.  Eventually the
shortest path $P$ will only go through supernodes of size one,
corresponding to a path in the original graph.  At this point we know
that $P$ must be a shortest path from $s$ to $t$ in the original
graph.  In the best case the optimal path in each iteration will be a
refinement of the optimal path from the previous iteration.  This
would result in $O(\log n)$ shortest paths computations, each in fairly
coarse graphs.  On the other hand, in the worst case CFDP will take
$\Omega(n)$ iterations to refine the whole graph, and many of the
iterations will involve finding shortest paths in large
graphs.  In this case CFDP takes $\Omega(kn^2)$ time which is
much worst than the standard dynamic programming approach.

Now suppose we use HA*LD to find the shortest path from $s$ to $t$ in
a graph like the one in Figure~\ref{fig:cfdp}(a).  We can build an
abstraction hierarchy with $O(\log n)$ levels where each supernode at
level $i$ contains $2^i$ nodes from one column of the original graph.
The coarse graph in Figure~\ref{fig:cfdp}(b) represents the highest
level of this abstraction hierarchy.  Note that HA*LD will consider a
small number, $O(\log n)$, of predefined graphs while CFDP can end up
considering a much larger number, $\Omega(n)$, of graphs.  In the best
case scenario HA*LD will expand only the nodes that are in the
shortest path from $s$ to $t$ at each level of the hierarchy.  In the
worst case HA*LD will compute a lightest path and context for every
node in the hierarchy (here a context for a node $v$ is a path from
$v$ to $t$).  At the $i$-th abstraction level we have a graph with
$O(kn/2^i)$ nodes and edges.  HA*LD will spend at most
$O(kn\log(kn)/2^i)$ time computing paths and contexts at level $i$.
Summing over levels we get at most $O(kn \log(kn))$ time total, which
is not much worst than the $O(kn)$ time taken by the standard dynamic
programming approach.

\begin{figure}
\centering 
\begin{tabular}{c}
\includegraphics[width=2.64in]{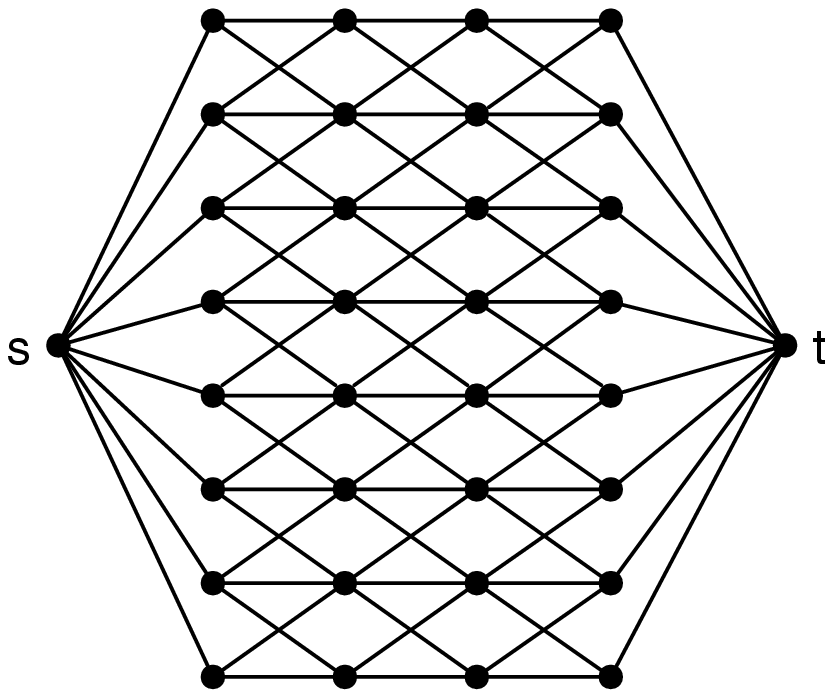} \\
(a)
\end{tabular}
\begin{tabular}{cc}
\includegraphics[width=2.64in]{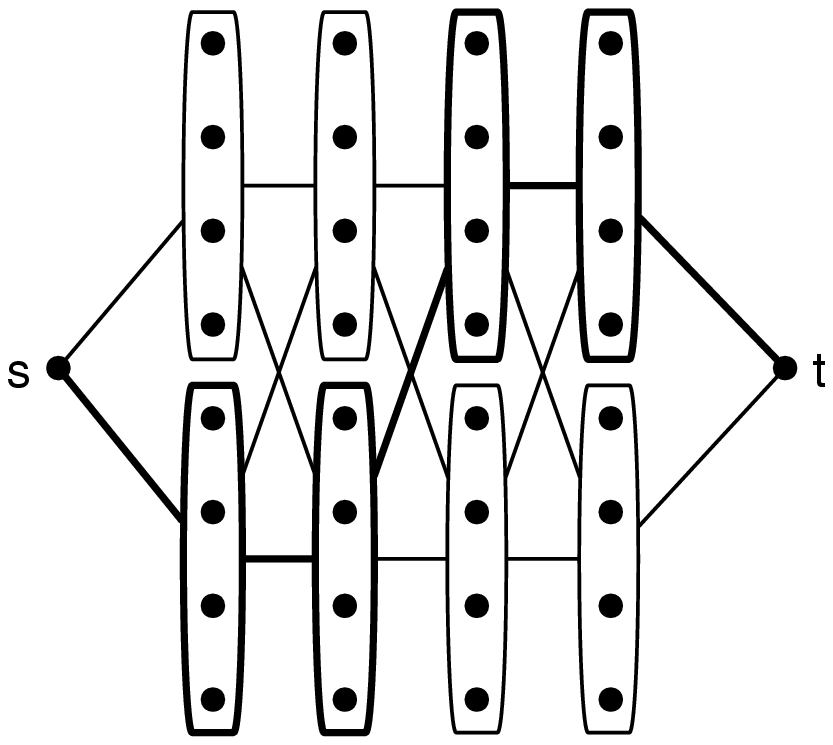} &
\includegraphics[width=2.64in]{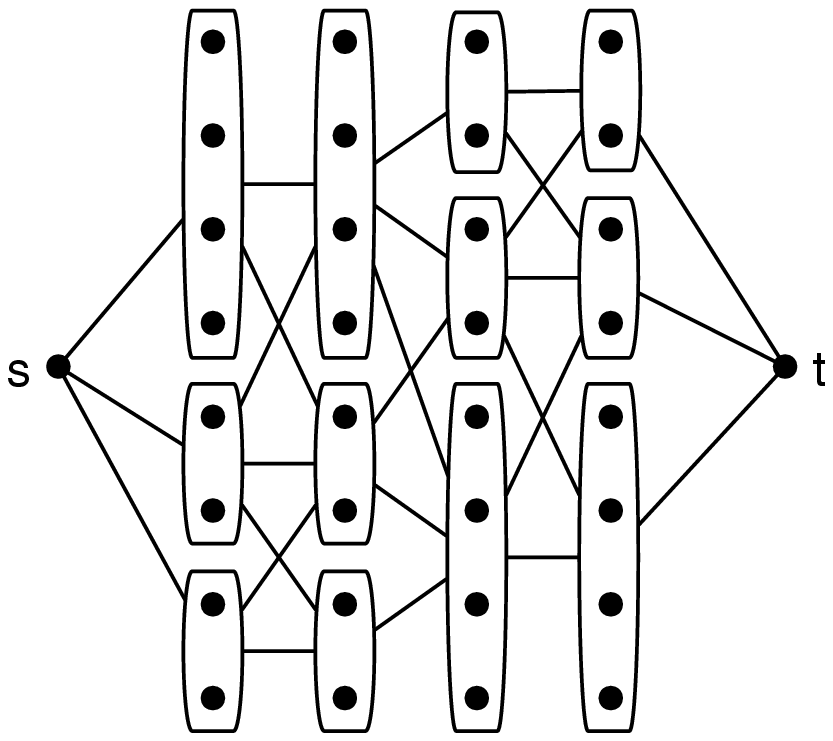} \\
(b) & (c)
\end{tabular}
\caption{(a) Original dynamic programming graph.  (b) Coarse graph
with shortest path shown in bold.  (c) Refinement of the coarse graph
along the shortest path.}
\label{fig:cfdp}
\end{figure}

\subsection{Hierarchical Heuristic Search}

Our hierarchical method is also related to the HA* and HIDA*
algorithms described by \citeA{Holte96} and \citeA{Holte05}.  These
methods are restricted to shortest paths problems but they also use a
hierarchy of abstractions.  A heuristic function is defined for each
level of abstraction using shortest paths to the goal at the level
above.  The main idea is to run A* or IDA* to compute a shortest path
while computing heuristic values on-demand.  Let $\abs$ map a node to
its abstraction and let $g$ be the goal node in the concrete graph.
Whenever the heuristic value for a concrete node $v$ is needed we call
the algorithm recursively to find the shortest path from $\abs(v)$ to
$\abs(g)$.  This recursive call uses heuristic values defined from a
further abstraction, computed through deeper recursive calls.

It is not clear how to generalize HA* and HIDA* to lightest derivation
problems that have rules with multiple antecedents.  Another
disadvantage is that these methods can potentially ``stall'' in the
case of directed graphs.  For example, suppose that when using HA* or
HIDA* we expand a node with two successors $x$ and $y$, where $x$ is
close to the goal but $y$ is very far.  At this point we need a
heuristic value for $x$ and $y$, and we might have to spend a long time
computing a shortest path from $\abs(y)$ to $\abs(g)$.  On the other
hand, HA*LD would not wait for this shortest path to be fully
computed.  Intuitively HA*LD would compute shortest paths from
$\abs(x)$ and $\abs(y)$ to $\abs(g)$ simultaneously.  As soon as the
shortest path from $\abs(x)$ to $\abs(g)$ is found we can start
exploring the path from $x$ to $g$, independent of how long it would
take to compute a path from $\abs(y)$ to $\abs(g)$.

\section{Convex Object Detection}
\label{sec:experiments}

Now we consider an application of HA*LD to the problem of detecting
convex objects in images.  We pose the problem using a formulation
similar to the one described by \citeA{Raphael01}, where the optimal
convex object around a point can be found by solving a shortest path
problem.  We compare HA*LD to other search methods, including CFDP and
A* with pattern databases.  The results indicate that HA*LD performs
better than the other methods over a wide range of inputs.

Let $x$ be a reference point inside a convex object.  We can represent
the object boundary using polar coordinates with respect to a
coordinate system centered at $x$.  In this case the object is
described by a periodic function $r(\theta)$ specifying the distance
from $x$ to the object boundary as a function of the angle $\theta$.
Here we only specify $r(\theta)$ at a finite number of angles
$(\theta_0,\ldots,\theta_{N-1})$ and assume the boundary is a straight
line segment between sample points.  We also assume the object is
contained in a ball of radius $R$ around $x$ and that $r(\theta)$ is
an integer.  Thus an object is parametrized by $(r_0,\ldots,r_{N-1})$
where $r_i \in [0,R-1]$.  An example with $N=8$ angles is shown in
Figure~\ref{fig:convex}.

\begin{figure}
\centering \includegraphics[width=3in]{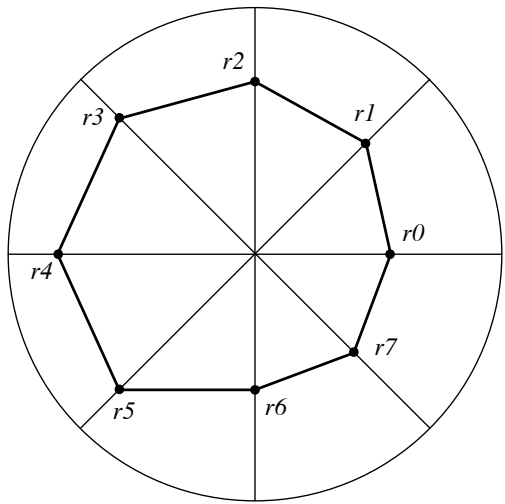}
\caption{A convex set specified by a hypothesis $(r_0,\ldots,r_7)$.}
\label{fig:convex}
\end{figure}

Not every hypothesis $(r_0,\ldots,r_{N-1})$ specifies a convex object.
The hypothesis describes a convex set exactly when the object boundary
turns left at each sample point $(\theta_i,r_i)$ as $i$ increases.
Let $C(r_{i-1},r_i,r_{i+1})$ be a Boolean function indicating when
three sequential values for $r(\theta)$ define a boundary that is
locally convex at $i$.  The hypothesis $(r_0,\ldots,r_{N-1})$ is
convex when it is locally convex at each $i$.\footnote{This
parametrization of convex objects is similar but not identical to the
one used by \citeA{Raphael01}.}

Throughout this section we assume that the reference point $x$ is
fixed in advance.  Our goal is to find an ``optimal'' convex object
around a given reference point.  In practice reference locations can
be found using a variety of methods such as a Hough transform.

Let $D(i,r_i,r_{i+1})$ be an image data cost measuring the evidence
for a boundary segment from $(\theta_i,r_i)$ to
$(\theta_{i+1},r_{i+1})$.  We consider the problem of finding a convex
object for which the sum of the data costs along the whole boundary is
minimal.  That is, we look for a convex hypothesis minimizing the
following energy function,
$$E(r_0,\ldots,r_{N-1}) = \sum_{i=0}^{N-1} D(i,r_i,r_{i+1}).$$

The data costs can be precomputed and specified by a lookup table with
$O(NR^2)$ entries.  In our experiments we use a data cost based on the
integral of the image gradient along each boundary segment.  Another
approach would be to use the data term described by \citeA{Raphael01} where the
cost depends on the contrast between the inside and the outside of the
object measured within the pie-slice defined by $\theta_i$ and
$\theta_{i+1}$.

An optimal convex object can be found using standard dynamic
programming techniques.  Let $B(i,r_0,r_1,r_{i-1},r_i)$ be the cost of
an optimal partial convex object starting at $r_0$ and $r_1$ and
ending at $r_{i-1}$ and $r_{i}$.  Here we keep track of the last two
boundary points to enforce the convexity constraint as we extend
partial objects.  We also have to keep track of the first two boundary
points to enforce that $r_N=r_0$ and the convexity constraint at
$r_0$.  We can compute $B$ using the recursive formula,
\begin{align*}
B(1,r_0,r_1,r_0,r_1) = & \; D(0,r_0,r_1), \\
B(i+1,r_0,r_1,r_i,r_{i+1}) = & \; \min_{r_{i-1}} B(i,r_0,r_1,r_{i-1},r_i)
+ D(i, r_i, r_{i+1}),
\end{align*}
where the minimization is over choices for $r_{i-1}$ such that
$C(r_{i-1},r_i,r_{i+1}) = \true$.  The cost of an optimal object is
given by the minimum value of $B(N,r_0,r_1,r_{N-1},r_0)$ such that
$C(r_{N-1},r_0,r_1) = \true$.  An optimal object can be found by
tracing-back as in typical dynamic programming algorithms.  The main
problem with this approach is that the dynamic programming table has
$O(NR^4)$ entries and it takes $O(R)$ time to compute each entry.  The
overall algorithm runs in $O(NR^5)$ time which is quite slow.

Now we show how optimal convex objects can be defined in terms of a
lightest derivation problem.  Let $\convex{i,r_0,r_1,r_{i-1},r_i}$
denote a partial convex object starting at $r_0$ and $r_1$ and ending
at $r_{i-1}$ and $r_{i}$.  This corresponds to an entry in the dynamic
programming table described above.  Define the set of statements,
$$\Sigma = \{\convex{i,a,b,c,d} \;|\; i \in [1,N],\; a,b,c,d \in
[0,R-1] \} \cup \{\goal\}.$$ An optimal convex object corresponds to
a lightest derivations of $\goal$ using the rules in
Figure~\ref{fig:convex-rules}.  The first set of rules specify the
cost of a partial object from $r_0$ to $r_1$.  The second set of rules
specify that an object ending at $r_{i-1}$ and $r_i$ can be extended
with a choice for $r_{i+1}$ such that the boundary is locally convex
at $r_i$.  The last set of rules specify that a complete convex object
is a partial object from $r_0$ to $r_N$ such that $r_N = r_0$ and the
boundary is locally convex at $r_0$.

\begin{figure}
(1) for $r_0,r_1 \in [0,R-1]$,

\parbox{.2cm}{\hspace{.2cm}}
\parbox{2cm}{\irule{}{}{}
{\con{\convex{1,r_0,r_1,r_0,r_1} = D(0,r_0,r_1)}}}

\vspace{.2cm}
(2) for $r_0,r_1,r_{i-1},r_i,r_{i+1} \in [0,R-1]$ such that
$C(r_{i-1},r_i,r_{i+1}) = \true$,

\parbox{.2cm}{\hspace{.2cm}}
\parbox{2in}{\irule{}
{\ant{\convex{i,r_0,r_1,r_{i-1},r_i} = w}}
{}
{\con{\convex{i+1,r_0,r_1,r_i,r_{i+1}} = w + D(i,r_i,r_{i+1})}}}

\vspace{.2cm}
(3) for $r_0,r_1,r_{N-1} \in [0,R-1]$ such that
$C(r_{N-1},r_0,r_1) = \true$,

\parbox{.2cm}{\hspace{.2cm}}
\parbox{2in}{\irule{}
{\ant{\convex{N,r_0,r_1,r_{N-1},r_0} = w}}
{}
{\con{\goal = w}}}
\caption{Rules for finding an optimal convex object.}
\label{fig:convex-rules}
\end{figure}

To construct an abstraction hierarchy we define $L$ nested partitions
of the radius space $[0,R-1]$ into ranges of integers.  In an abstract
statement instead of specifying an integer value for $r(\theta)$ we
will specify the range in which $r(\theta)$ is contained.  To simplify
notation we assume that $R$ is a power of two.  The $k$-th partition
$P^k$ contains $R/2^k$ ranges, each with $2^k$ consecutive integers.
The $j$-th range in $P^k$ is given by $[j * 2^k,(j+1) * 2^k - 1]$.

The statements in the abstraction hierarchy are,
$$\Sigma_k = \; \{\convex{i,a,b,c,d} \;|\; i \in [1,N],\; a,b,c,d \in
P^k\} \cup \{\goal_k\},$$ for $k \in [0,L-1]$.  A range in $P^0$
contains a single integer so $\Sigma_0 = \Sigma$.  Let $f$ map a range
in $P^k$ to the range in $P^{k+1}$ containing it.  For statements in
level $k < L-1$ we define the abstraction function,
\begin{align*}
\abs(\convex{i,a,b,c,d}) = & \; \convex{i,f(a),f(b),f(c),f(d)}, \\
\abs(\goal_k) = & \; \goal_{k+1}.
\end{align*}

The abstract rules use bounds on the data costs for boundary segments
between $(\theta_i,s_i)$ and $(\theta_{i+1},s_{i+1})$ where $s_i$ and
$s_{i+1}$ are ranges in $P^k$,
$$D^k(i,s_i,s_{i+1}) = 
\min_{\begin{array}{c}
r_i \in s_i \\ r_{i+1} \in s_{i+1}
\end{array}} D(i,r_i,r_{i+1}).$$
Since each range in $P^k$ is the union of two ranges in $P^{k-1}$ one
entry in $D^k$ can be computed quickly (in constant time) once
$D^{k-1}$ is computed.  The bounds for all levels can be computed in
$O(NR^2)$ time total.  We also need abstract versions of the convexity
constraints.  For $s_{i-1}, s_i, s_{i+1} \in P^k$, let
$C^k(s_{i-1},s_i,s_{i+1}) = \true$ if there exist integers $r_{i-1}$,
$r_i$ and $r_{i+1}$ in $s_{i-1}$, $s_i$ and $s_{i+1}$ respectively
such that $C(r_{i-1}, r_i, r_{i+1}) = \true$.  The value of $C^k$ can
be defined in closed form and evaluated quickly using simple geometry.

The rules in the abstraction hierarchy are almost identical to the
rules in Figure~\ref{fig:convex-rules}.  The rules in level $k$ are
obtained from the original rules by simply replacing each instance of
$[0,R-1]$ by $P^k$, $C$ by $C^k$ and $D$ by $D^k$.

\subsection{Experimental Results}

Figure~\ref{fig:cells} shows an example image with a set of reference
locations that we selected manually and the optimal convex object
found around each reference point.  There are 14 reference locations
and we used $N=30$ and $R=60$ to parametrize each object.
Table~\ref{tab:time} compares the running time of different
optimization algorithms we implemented for this problem.  Each line
shows the time it took to solve all 14 problems contained in the
example image using a particular search algorithm.  The standard DP
algorithm uses the dynamic programming solution outlined above.  The
CFDP method is based on the algorithm by \citeA{Raphael01} but modified
for our representation of convex objects.  Our hierarchical A*
algorithm uses the abstraction hierarchy described here.  For A* with
pattern databases we used dynamic programming to compute a pattern
database at a particular level of abstraction, and then used this
database to provide heuristic values for A*.  Note
that for the problem described here the pattern database depends on
the input.  The running times listed include the time it took to
compute the pattern database in each case.

We see that CFDP, HA*LD and A* with pattern databases are much more
efficient than the standard dynamic programming algorithm that does
not use abstractions.  HA*LD is slightly faster then the other methods
in this example.  Note that while the running time varies from
algorithm to algorithm the output of every method is the same as they
all find globally optimum objects.

\begin{figure}
\centering 
\begin{tabular}{c}
\includegraphics[width=4.5in]{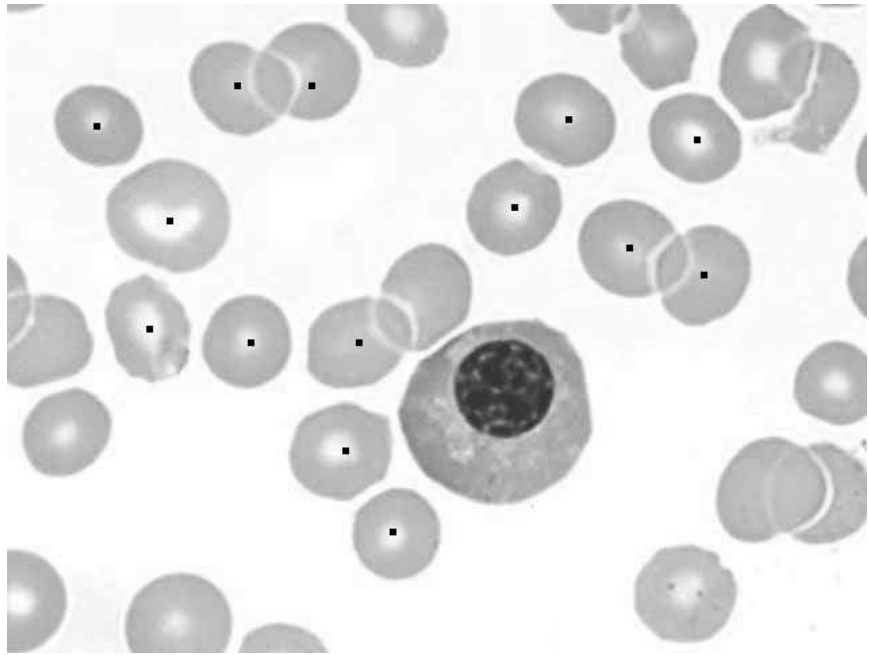} \\ (a) \\ \\
\includegraphics[width=4.5in]{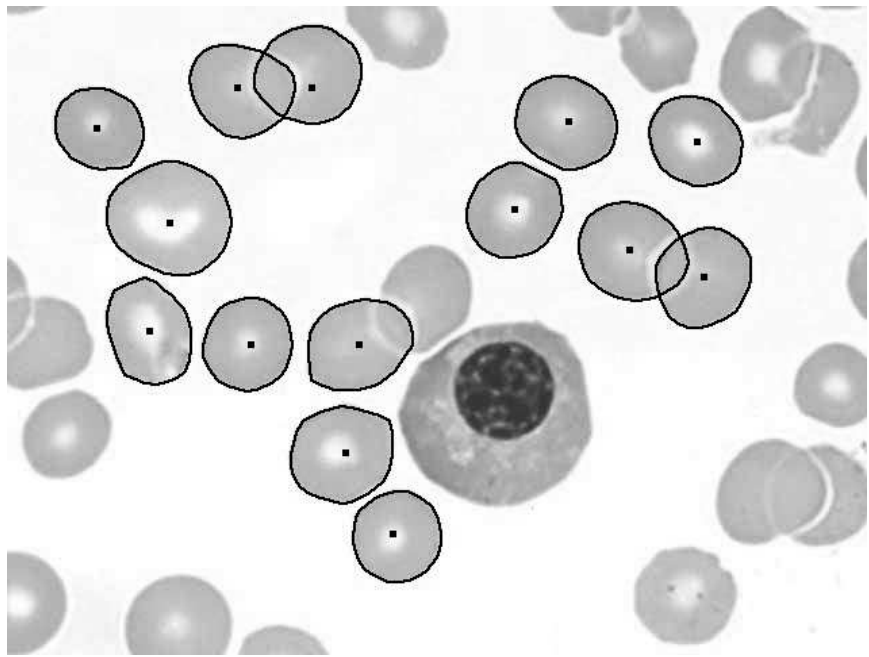} \\ (b) 
\end{tabular}
\caption{(a) Reference locations.  (b) Optimal convex objects.}
\label{fig:cells}
\end{figure}

\begin{table}
\centering
\begin{tabular}{|c|c|}
\hline 
Standard DP & 6718.6 seconds \\ 
\hline 
CFDP & 13.5 seconds \\ 
\hline
HA*LD & 8.6 seconds \\ 
\hline
A* with pattern database in $\Sigma_2$ & 14.3 seconds \\
\hline
A* with pattern database in $\Sigma_3$ & 29.7 seconds \\
\hline
\end{tabular}
\caption{Running time comparison for the example in Figure~\ref{fig:cells}.}
\label{tab:time}
\end{table}

For a quantitative evaluation of the different search algorithms we
created a large set of problems of varying difficulty and size as
follows.  For a given value of $R$ we generated square images of width
and height $2*R+1$.  Each image has a circle with radius less than $R$
near the center and the pixels in an image are corrupted by
independent Gaussian noise.  The difficulty of a problem is controlled
by the standard deviation, $\sigma$, of the noise.
Figure~\ref{fig:circle} shows some example images and optimal convex
object found around their centers.

\begin{figure}
\centering 
\begin{tabular}{ccc}
\includegraphics[width=1.5in]{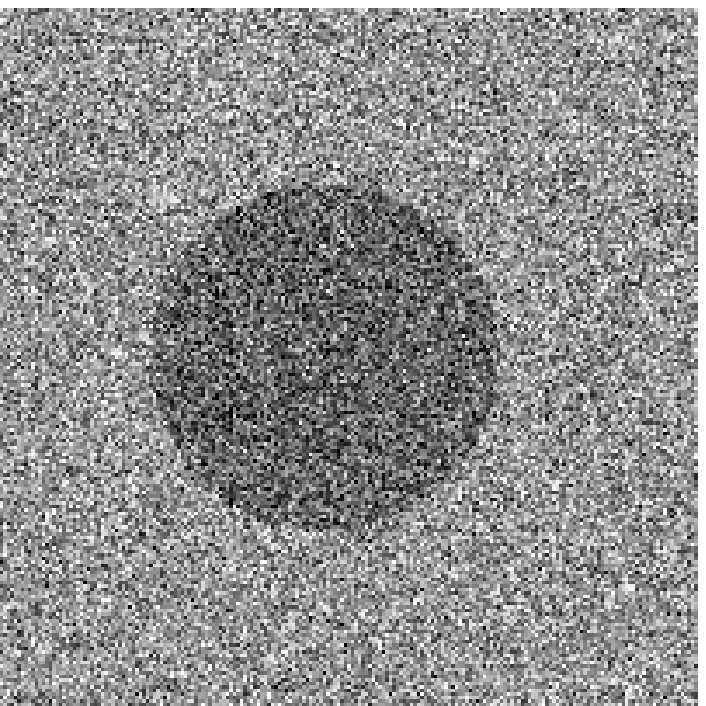} &
\includegraphics[width=1.5in]{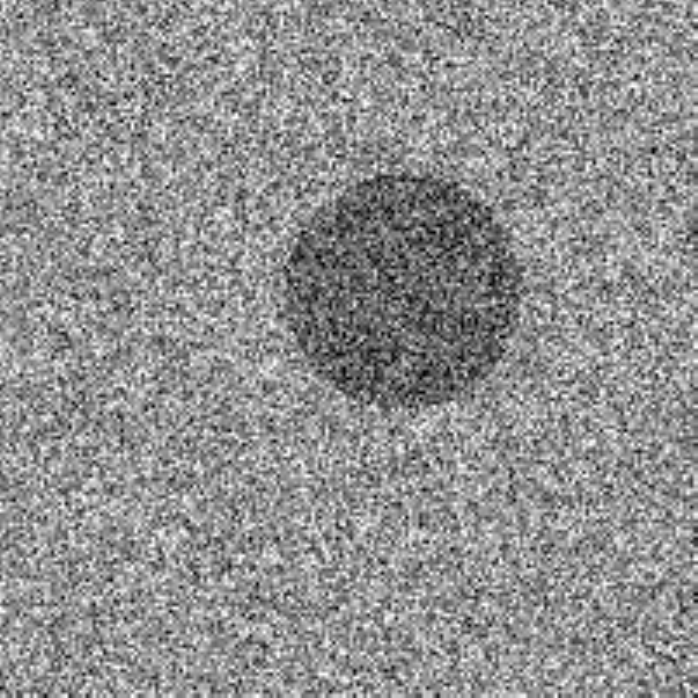} &
\includegraphics[width=1.5in]{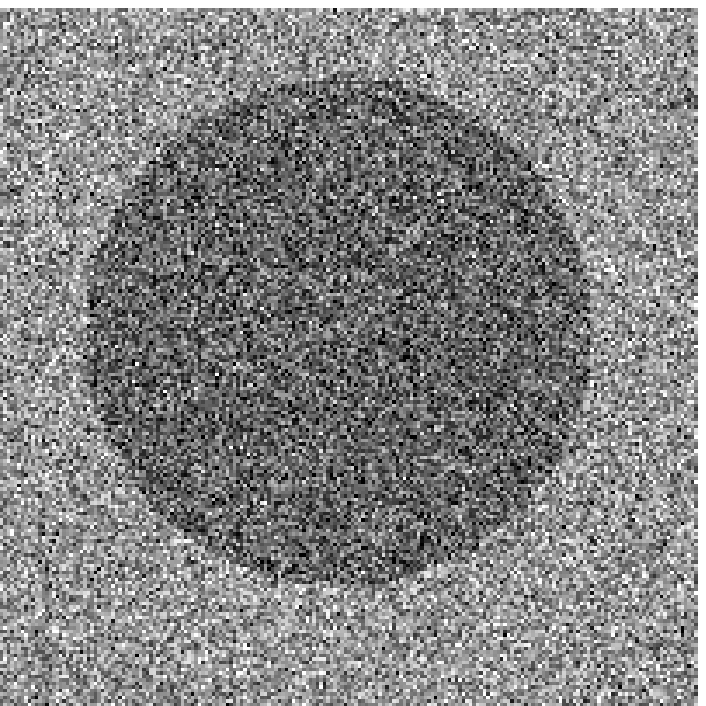} \\
\includegraphics[width=1.5in]{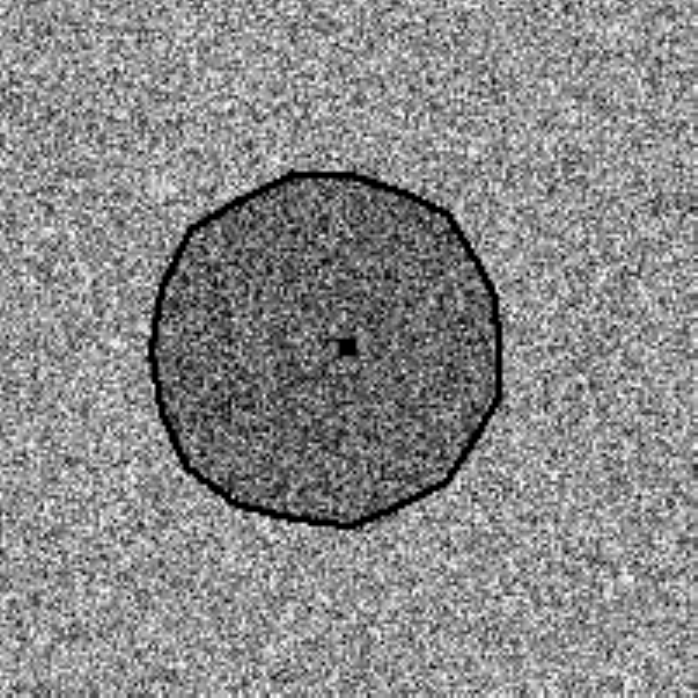} &
\includegraphics[width=1.5in]{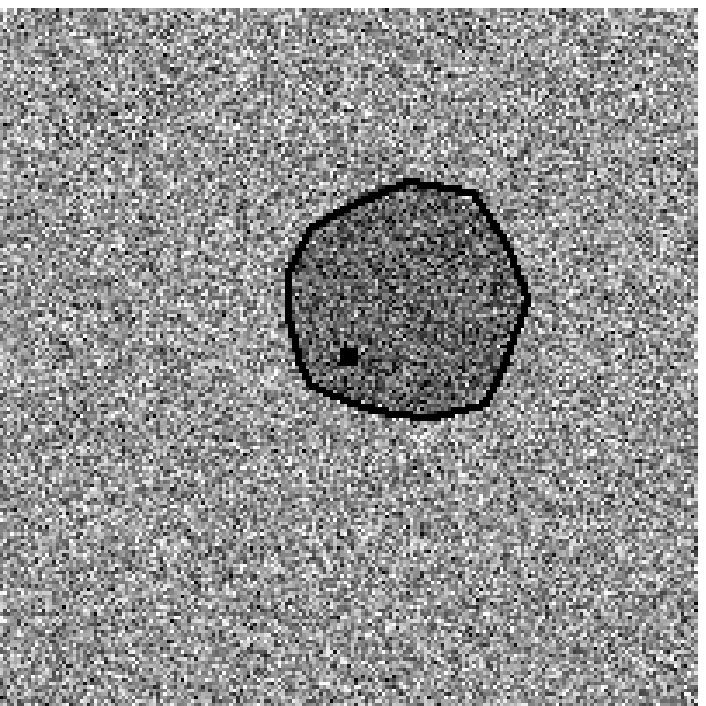} &
\includegraphics[width=1.5in]{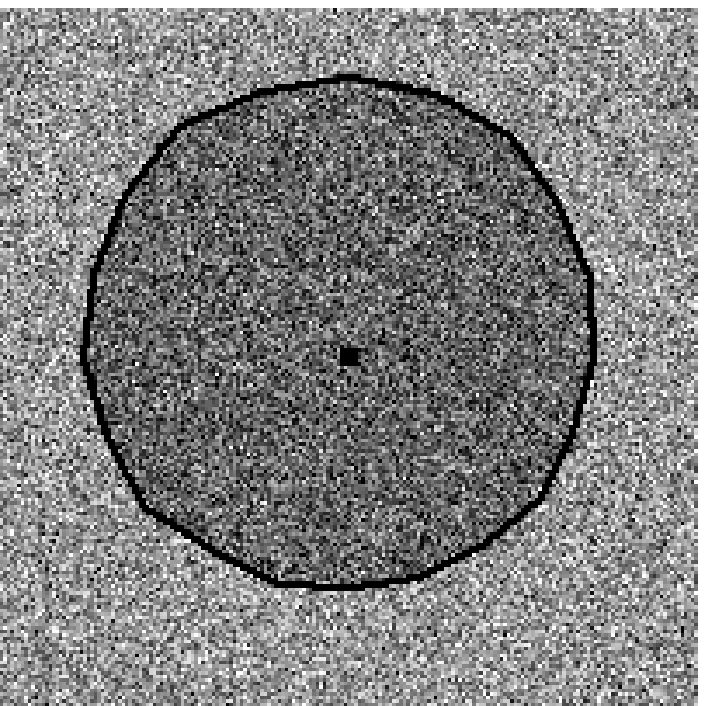} 
\end{tabular}
\caption{Random images with circles and the optimal convex object
around the center of each one (with $N=20$ and $R=100$).  The noise
level in the images is $\sigma=50$.}
\label{fig:circle}
\end{figure}

The graph in Figure~\ref{fig:noise} shows the running time (in
seconds) of the different search algorithms as a function of the noise
level when the problem size is fixed at $R=100$.  Each sample point
indicates the average running time over 200 random inputs.  The graph
shows running times up to a point after which the circles can not be
reliably detected.  We compared HA*LD with CFDP and A* using
pattern databases (PD2 and PD3).  Here PD2 and PD3 refer to A* with a
pattern database defined in $\Sigma_2$ and $\Sigma_3$ respectively.
Since the pattern database needs to be recomputed for each input there
is a trade-off in the amount of time spent computing the database and
the accuracy of the heuristic it provides.  We see that for easy
problems it is better to use a smaller database (defined at a higher
level of abstraction) while for harder problems it is worth spending
time computing a bigger database.  HA*LD outperforms the other methods
in every situation captured here.

\begin{figure}
\centering 
\includegraphics[width=5.5in]{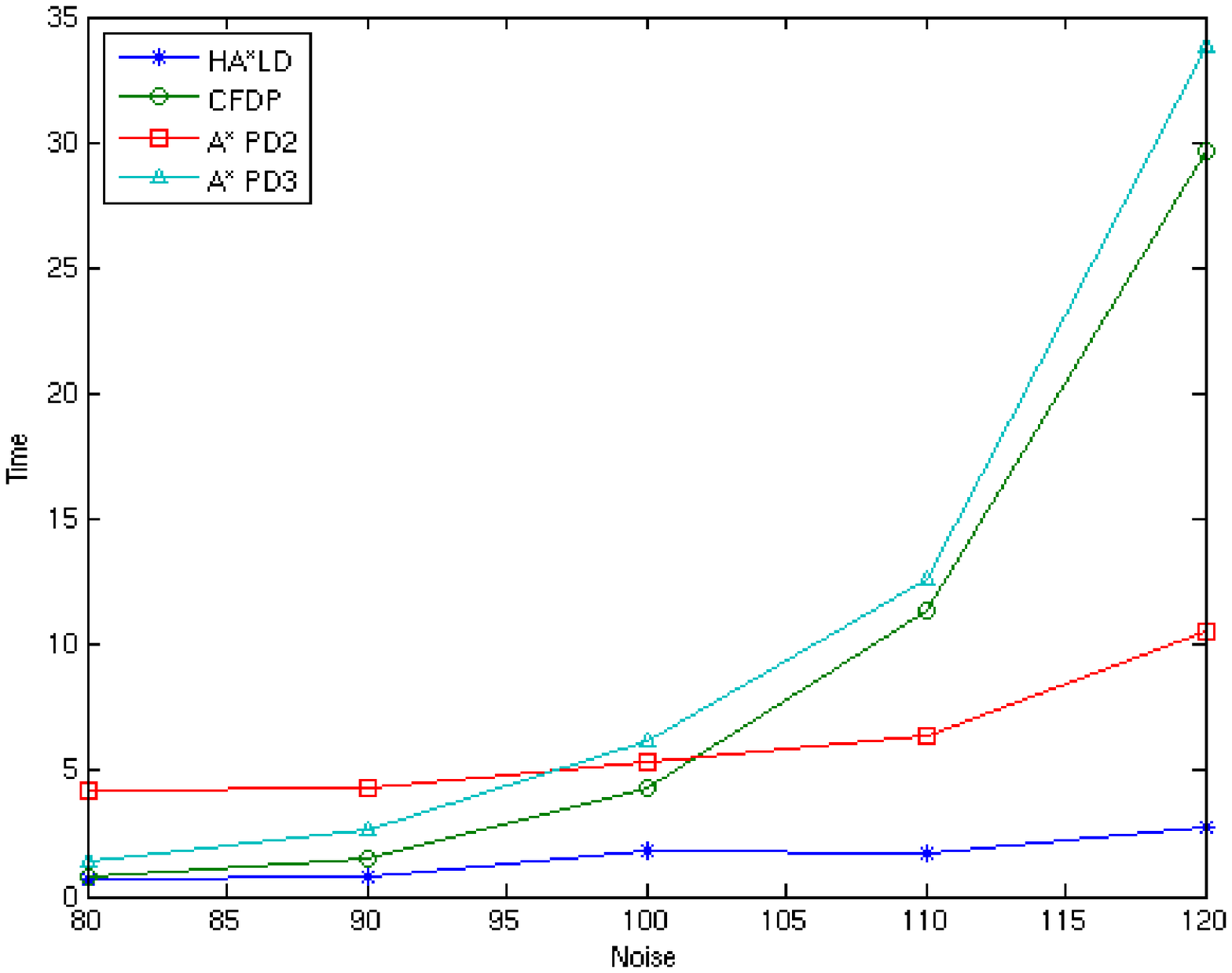}
\caption{Running time of different search algorithms as a function of
the noise level $\sigma$ in the input.  Each sample point indicates
the average running time taken over 200 random inputs.  In each case
$N=20$ and $R=100$.  See text for discussion.}
\label{fig:noise}
\end{figure}

\begin{figure}
\centering 
\includegraphics[width=5.5in]{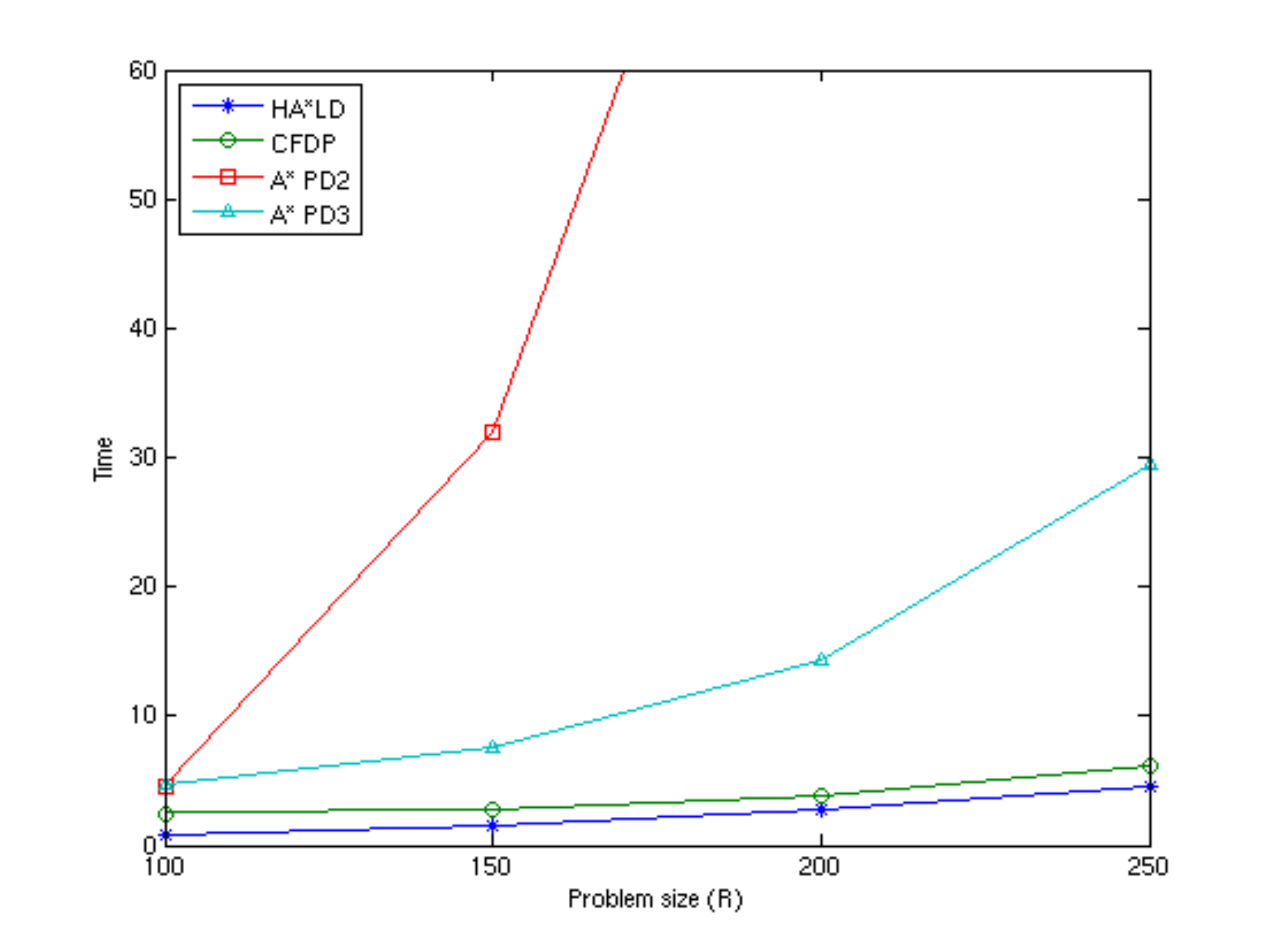}
\caption{Running time of different search algorithms as a function of
the problem size $R$.  Each sample point indicates the average running
time taken over 200 random inputs.  In each case $N=20$ and
$\sigma=100$.  See text for discussion.}
\label{fig:size}
\end{figure}

Figure~\ref{fig:size} shows the running time of the different
methods as a function of the problem size $R$, on problems with a
fixed noise level of $\sigma=100$.  As before each sample point
indicates the average running time taken over 200 random inputs.  We
see that the running time of the pattern database approach grows
quickly as the problem size increases.  This is because computing the
database at any fixed level of abstraction takes $O(NR^5)$ time.  On
the other hand the running time of both CFDP and HA*LD grows much
slower.  While CFDP performed essentially as well as HA*LD in this
experiment, the graph in Figure~\ref{fig:noise} shows that HA*LD
performs better as the difficulty of the problem increases.

\section{Finding Salient Curves in Images}
\label{sec:curves}

A classical problem in computer vision involves finding salient curves
in images.  Intuitively the goal is to find long and smooth curves
that go along paths with high image gradient.  The standard way to
pose the problem is to define a saliency score and search for curves
optimizing that score.  Most methods use a score defined by a
simple combination of local terms.  For example, the score usually depends
on the curvature and the image gradient at each point of a curve.
This type of score can often be optimized efficiently using dynamic
programming or shortest paths algorithms
\cite{Montanari71,Shashua88,Basri96,Williams96}.

Here we consider a new compositional model for finding salient
curves.  An important aspect of this model is that it can capture
global shape constraints.  In particular, it looks for curves that are
almost straight, something that can not be done using local constraints
alone.  Local constraints can enforce small curvature at each point of
a curve, but this is not enough to prevent curves from turning and
twisting around over long distances.  The problem of finding the most
salient curve in an image with the compositional model defined here
can be solved using dynamic programming, but the approach is too slow
for practical use.  Shortest paths algorithms are not applicable
because of the compositional nature of the model.  Instead we can use
A*LD with a heuristic function derived from an abstraction (a pattern
database).

\begin{figure}
\centering
\includegraphics[width=2.9in]{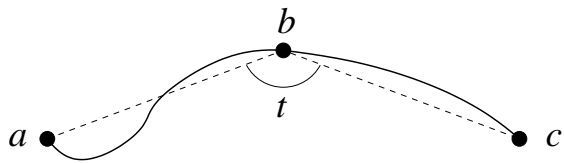}
\caption{A curve with endpoints $(a,c)$ is formed by composing curves
  with endpoints $(a,b)$ and $(b,c)$.  We assume that $t \ge \pi/2$.
  The cost of the composition is proportional to $\sin^2(t)$.  This
  cost is scale invariant and encourages curves to be relatively
  straight.  }
\label{fig:curve}
\end{figure}

Let $C_1$ be a curve with endpoints $a$ and $b$ and $C_2$ be a curve
with endpoints $b$ and $c$.  The two curves can be composed to form a
curve $C$ from $a$ to $c$.  We define the weight of the composition to
be the sum of the weights of $C_1$ and $C_2$ plus a shape cost that
depends on the geometric arrangement of points $(a, b, c)$.
Figure~\ref{fig:curve} illustrates the idea and the shape costs we
use.  Note that when $C_1$ and $C_2$ are long, the arrangement of
their endpoints reflect non-local geometric properties.  In general we
consider composing $C_1$ and $C_2$ if the angle formed by
$\overline{ab}$ and $\overline{bc}$ is at least $\pi/2$ and the
lengths of $C_1$ and $C_2$ are approximately equal.  These constraints
reduce the total number of compositions and play an important role in
the abstract problem defined below.

Besides the compositional rule we say that if $a$ and $b$ are nearby
locations, then there is a short curve with endpoints $a$ and $b$.
This forms a base case for creating longer curves.  We assume that
these short curves are straight, and their weight depends only on the
image data along the line segment from $a$ to $b$.  We use a data
term, $\seg{a,b}$, that is zero if the image gradient along pixels in
$\overline{ab}$ is perpendicular to $\overline{ab}$, and higher otherwise.

Figure~\ref{fig:curve-rules} gives a formal definition of the two
rules in our model.  The constants $k_1$ and $k_2$ specify the minimum
and maximum length of the base case curves, while $L$ is a constant
controlling the maximum depth of derivations.  A derivation of
$\curve{a,b,i}$ encodes a curve from $a$ to $b$.  The value $i$ can
be seen as an approximate measure of arclength.  A derivation of
$\curve{a,b,i}$ is a full binary tree of depth $i$ that encodes a
curve with length between $2^i k_1$ and $2^i k_2$.  
We let $k_2 = 2 k_1$ to allow for curves of any length.

\begin{figure}
(1) for pixels $a,b,c$ where the angle between $\overline{ab}$ and $\overline{bc}$ 
is at least $\pi/2$ and for $0 \le i \le L$,

\parbox{.2cm}{\hspace{.2cm}}
\parbox{2cm}{\irule{}
{\ant{\curve{a,b,i} = w_1}
 \ant{\curve{b,c,i} = w_2}}{}
{\con{\curve{a,c,i+1} = w_1 + w_2 + \shape{a,b,c}}}}

\vspace{.2cm}
(2) for pixels $a,b$ with $k_1 \le ||a-b|| \le k_2$,

\parbox{.2cm}{\hspace{.2cm}}
\parbox{2cm}{\irule{}{}{} 
{\con{\curve{a,b,0} = \seg{a,b}}}}
\caption{Rules for finding ``almost straight'' curves between a pair
of endpoints.  Here $L$, $k_1$ and $k_2$ are constants, while
$\shape{a,b,c}$ is a function measuring the cost of a composition.}
\label{fig:curve-rules}
\end{figure}

The rules in Figure~\ref{fig:curve-rules} do not define a good measure
of saliency by themselves because they always prefer short curves over
long ones.  We can define the saliency of a curve in terms of its
weight \emph{minus} its arclength, so that salient curves will be
light and long.  Let $\lambda$ be a positive constant.  We consider
finding the lightest derivation of $\goal$ using,

\onerule{\irule{}
{\ant{\curve{a,b,i} = w}}
{}
{\con{\goal = w - \lambda 2^i}}}

For an $n \times n$ image there are $\Omega(n^4)$ statements of the
form $\curve{a,c,i}$.  Moreover, if $a$ and $c$ are far apart there
are $\Omega(n)$ choices for a ``midpoint'' $b$ defining the two curves
that are composed in a lightest derivation of $\curve{a,c,i}$.  This
makes a dynamic programming solution to the lightest derivation
problem impractical.  We have tried using KLD but even for small
images the algorithm runs out of memory after a few
minutes.  Below we describe an abstraction we have used to define a
heuristic function for A*LD.

Consider a hierarchical set of partitions of an image into boxes.
The $i$-th partition is defined by tiling the image into boxes of $2^i
\times 2^i$ pixels.  The partitions form a pyramid with boxes of
different sizes at each level.  Each box at level $i$ is the union of
4 boxes at the level below it, and the boxes at level 0 are the pixels
themselves.  Let $f_i(a)$ be the box containing $a$ in the $i$-th
level of the pyramid.  Now define
$$\abs(\curve{a,b,i}) = \curve{f_i(a), f_i(b), i}.$$
Figure~\ref{fig:pyramid} illustrates how this map selects a pyramid
level for an abstract statement.  Intuitively $\abs$ defines an
adaptive coarsening criteria.  If $a$ and $b$ are far from each other,
a curve from $a$ to $b$ must be long, which in turn implies that we
map $a$ and $b$ to boxes in a coarse partition of the image.
This creates an abstract problem that
has a small number of statements without losing too much information.

To define the abstract problem we also need to define a set of
abstract rules.  Recall that for every concrete rule $r$ we need a
corresponding abstract rule $r'$ where the weight of $r'$ is at most
the weight of $r$.  There are a small number of rules with no
antecedents in Figure~\ref{fig:curve-rules}.  For each concrete rule
$\rightarrow_{\seg{a,b}} \curve{a,b,0}$ we define a corresponding
abstract rule, $\rightarrow_{\seg{a,b}} \abs(\curve{a,b,0})$.  The
compositional rules from Figure~\ref{fig:curve-rules} lead 
to abstract rules for composing curves between boxes,
$$\curve{A,B,i},\curve{B,C,i} \rightarrow_v \curve{A',C',i+1},$$ where
$A$, $B$ and $C$ are boxes at the $i$-th pyramid level while $A'$ and
$C'$ are the boxes at level $i+1$ containing $A$ and $C$ respectively.
The weight $v$ should be at most $\shape{a,b,c}$ where $a$, $b$ and
$c$ are arbitrary pixels in $A$, $B$ and $C$ respectively.  We compute
a value for $v$ by bounding the orientations of the line segments
$\overline{ab}$ and $\overline{bc}$ between boxes.

\begin{figure}
\centerline{\includegraphics[scale=0.5]{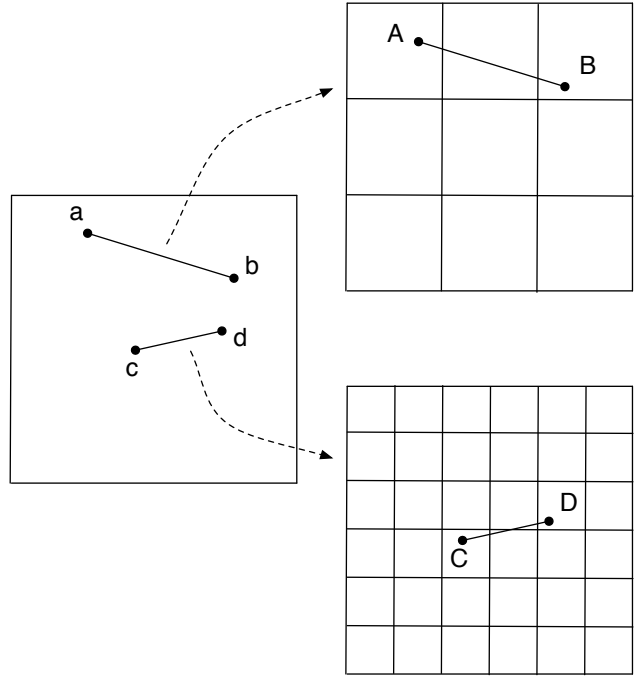}}
\caption{The abstraction maps each curve statement to a statement
about curves between boxes.  If $i > j$ then $\curve{a,b,i}$ gets
coarsened more than $\curve{c,d,j}$.  Since light curves are almost
straight, $i > j$ usually implies that $||a-b|| > ||c-d||$. }
\label{fig:pyramid}
\end{figure}

\begin{figure}
\centerline{
\includegraphics[scale=1]{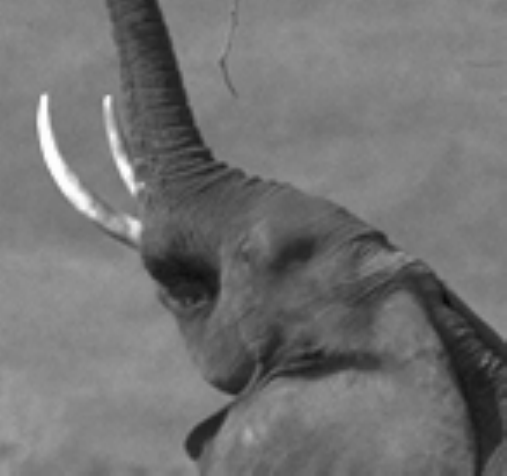}
\includegraphics[scale=1]{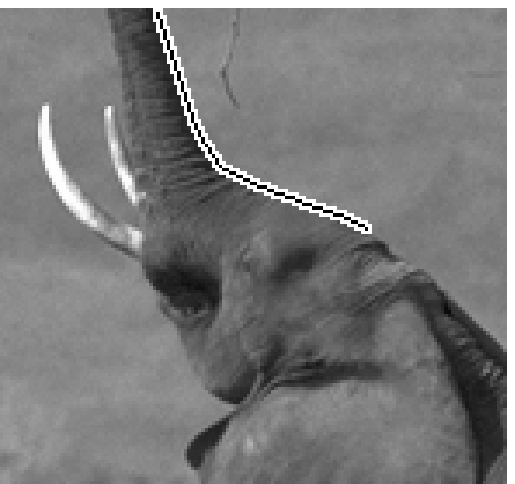}}
\centerline{$146 \times 137$ pixels.  Running time: 50 seconds (38 + 12).}
\vspace{.5cm}
\centerline{
\includegraphics[scale=1]{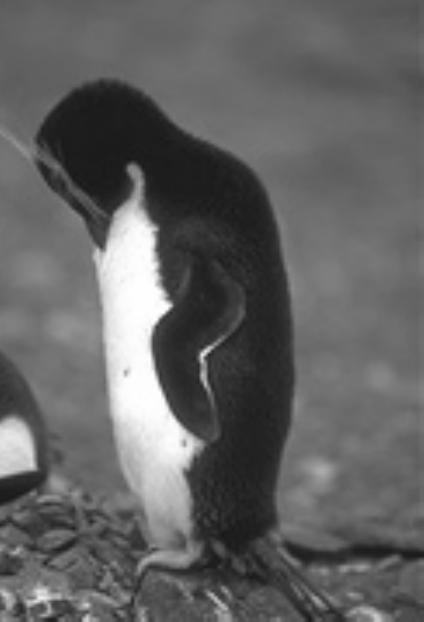}
\includegraphics[scale=1]{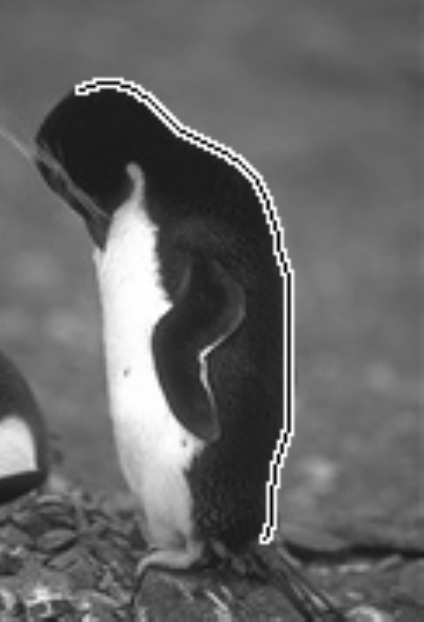}}
\centerline{$122 \times 179$ pixels.  Running time: 65 seconds (43 + 22).}
\vspace{.5cm}
\centerline{
\includegraphics[scale=1]{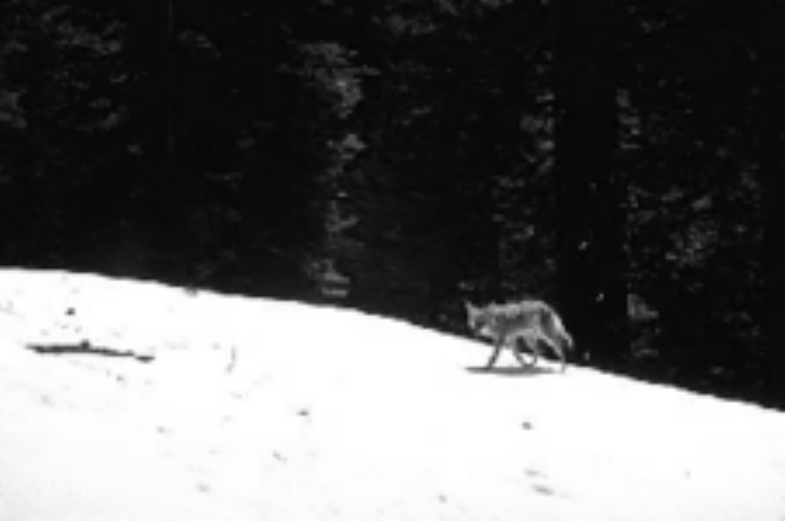}
\includegraphics[scale=1]{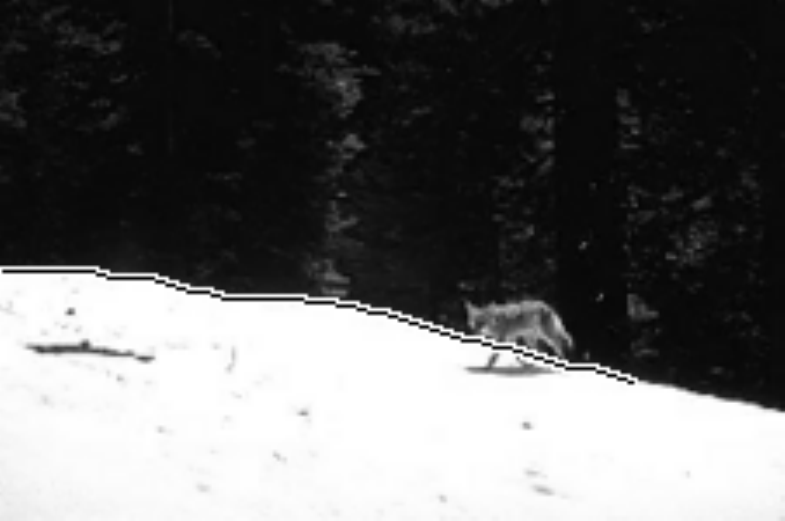}}
\centerline{$226 \times 150$ pixels.  Running time: 73 seconds (61 + 12).}
\caption{The most salient curve in different images.  The running time
is the sum of the time spent computing the pattern database and the
time spent solving the concrete problem.}
\label{fig:curves-results}
\end{figure}

\begin{figure}
\centerline{\includegraphics[scale=1]{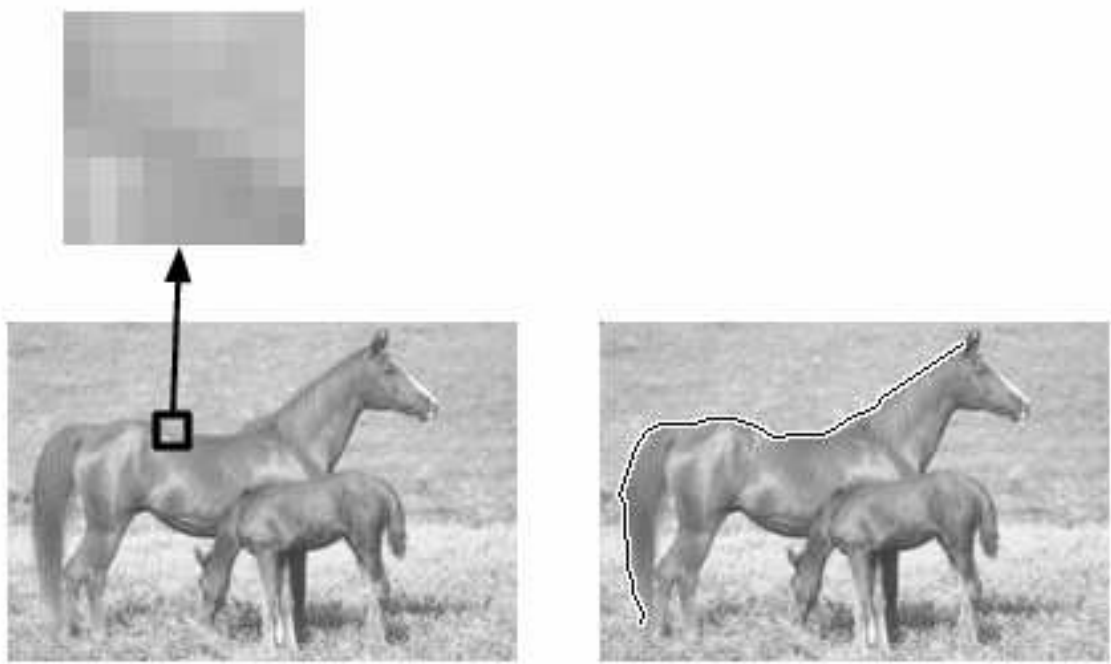}}
\caption{An example where the most salient curve goes over locations with
essentially no local evidence for a the curve at those locations.}
\label{fig:horse}
\end{figure}

The abstract problem defined above is relatively small even in large
images, so we can use the pattern database approach outlined in
Section~\ref{sec:abstractions}.  For each input image we use KLD to
compute lightest context weights for every abstract statement.  We
then use these weights as heuristic values for solving the concrete
problem with A*LD.  Figure~\ref{fig:curves-results} illustrates some
of the results we obtained using this method.  It seems like the
abstract problem is able to capture that most short curves can not be
extended to a salient curve.  It took about one minute to find the
most salient curve in each of these images.
Figure~\ref{fig:curves-results} lists the dimensions of each image and
the running time in each case.

Note that our algorithm does not rely on an initial binary edge
detection stage.  Instead the base case rules allow for salient curves
to go over any pixel, even if there is no local evidence for a
boundary at a particular location.  Figure~\ref{fig:horse} shows an
example where this happens.  In this case there is a small part of the
horse back that blends with the background if we consider local 
properties alone.

The curve finding algorithm described in this section would be very
difficult to formulate without A*LD and the general notion of
heuristics derived from abstractions for lightest derivation problems.
However, using the framework introduced in this paper it becomes relatively
easy to specify the algorithm.  

In the future we plan to ``compose'' the rules for computing salient
curves with rules for computing more complex structures.  The basic
idea of using a pyramid of boxes for defining an abstract problem should
be applicable to a variety of problems in computer vision.

\section{Conclusion}
\label{sec:conclusion}

Although we have presented some preliminary results in the last two
sections, we view the main contribution of this paper as providing a
general architecture for perceptual inference.  Dijkstra's shortest
paths algorithm and A* search are both fundamental algorithms with
many applications.  Knuth noted the generalization of Dijkstra's
algorithm to more general problems defined by a set of recursive
rules.  In this paper we have given similar generalizations for A*
search and heuristics derived from abstractions.  We have also
described a new method for solving lightest derivation problems using
a hierarchy of abstractions.  Finally, we have outlined an approach
for using these generalizations in the construction of processing
pipelines for perceptual inference.  

\section*{Acknowledgments}

This material is based upon work supported by the National Science
Foundation under Grant No. 0535174 and 0534820.


\bibliography{astar}
\bibliographystyle{jair/theapa}

\end{document}